\documentclass[english]{article}
\pdfoutput=1
\usepackage{fullpage}

\usepackage{color}
\usepackage[usenames,dvipsnames]{xcolor}
\usepackage[disable,backgroundcolor = White,textwidth=\marginparwidth]{todonotes}
\newcommand{\todoc}[2][]{\todo[color=Apricot!20,size=\tiny,#1]{C: #2}}

\newcommand{\todoa}[2][]{\todo[color=Purple!20,size=\tiny,#1]{A: #2}}
\newcommand{\todor}[2][]{\todo[color=Blue!10,size=\tiny,#1]{R: #2}}

\usepackage{comment}
\usepackage{babel}
\usepackage[round]{natbib}
\usepackage[utf8]{inputenc} % allow utf-8 input
\usepackage[T1]{fontenc}    % use 8-bit T1 fonts
\usepackage{lmodern}
\usepackage{url}            % simple URL typesetting
\usepackage[bookmarks=false]{hyperref}       % hyperlinks
\usepackage{booktabs}       % professional-quality tables
\usepackage{amsfonts}       % blackboard math symbols
\usepackage{nicefrac}       % compact symbols for 1/2, etc.
\usepackage{microtype}      % microtypography
\usepackage{enumerate}
\usepackage{algorithm}
\usepackage{algcompatible}
\usepackage{amsmath}
\usepackage{amsthm}
\usepackage{amssymb}
\usepackage{cancel}
\usepackage{epstopdf}
\usepackage{esvect}
\usepackage{graphicx,wrapfig,lipsum,framed}

\newcommand{\cA}{\mathcal{A}}
\newcommand{\cB}{\mathcal{B}}
\newcommand{\cE}{\mathcal{E}}
\newcommand{\cG}{\mathcal{G}}

\newcommand{\cZ}{\mathcal{Z}}
\newcommand{\cW}{\mathcal{W}}
\newcommand{\cM}{\mathcal{M}}
\newcommand{\cH}{\mathcal{H}}
\newcommand{\cF}{\mathcal{F}}
\newcommand{\cL}{\mathcal{L}}

\newcommand{\hP}{\hat{P}}

\newcommand{\Exp}[1]{\mathbb{E}\left[ #1 \right]} 
\newcommand{\Expc}[2]{\mathbb{E}\left[ \left. #1 \right| #2 \right]} 
\newcommand{\ind}{\mathbb{I}}
\newcommand{\one}[1]{\ind\left(#1\right)}
\newcommand{\seto}[1]{\left\{#1\right\}}
\newcommand{\real}{\mathbb{R}}
\newcommand{\R}{\mathbb{R}}
\newcommand{\bS}{\mathbb{S}}
\newcommand{\inpro}[2]{\langle #1, #2\rangle}

\newcommand{\ip}[1]{\langle#1\rangle}
\newcommand{\inner}[1]{\left\langle#1\right\rangle}

\newcommand{\set}[2]{\left\{#1 \,\vert\, #2 \right\}}
\newcommand{\lt}{\ell_t}
\newcommand{\ttheta}{\tilde{\theta}}

\newcommand{\htheta}{\hat{\theta}}

\newcommand{\norm}[1]{\left\| #1 \right\|}
\newcommand{\bd}{\mathrm{bd}}
\newcommand{\inangle}[2]{(#1,#2)}
\newcommand\numberthis{\addtocounter{equation}{1}\tag{\theequation}}
\newcommand{\uD}{\overline{D}}
\newcommand{\lD}{\underline{D}}
\newcommand{\ra}{\rightarrow}

\newcommand{\Prob}[1]{\mathbb{P}\left[#1\right]}
\newcommand{\Probc}[2]{\mathbb{P}\left[\left. #1 \, \right| #2\right]}
\DeclareMathOperator*{\argmin}{argmin}
\DeclareMathOperator*{\argmax}{argmax}
\DeclareMathOperator{\dom}{dom}
\DeclareMathOperator{\epi}{epi}
\usepackage[capitalize]{cleveref}

\newtheorem{thm}{Theorem}[section]
\newtheorem{lemma}[thm]{Lemma}
\newtheorem{prop}[thm]{Proposition}
\newtheorem{cor}[thm]{Corollary}
\newtheorem{example}[thm]{Example}

\title{
Following the Leader and
Fast Rates in Linear Prediction:
Curved Constraint Sets and Other Regularities\thanks{R. Huang and Cs. Szepesv\'ari are with the Department of Computing Science,
 University of Alberta, AB, Canada, email: \texttt{ruitong@ualberta.ca},  \texttt{szepesva@ualberta.ca}. T. Lattimore was with the 
School of Informatics and Computing, Indiana University, IN, USA, email: \texttt{tor.lattimore@gmail.com}. A. Gy\"orgy is
with the Department of Electrical and Electronic Engineering, Imperial College London, UK, email: \texttt{a.gyorgy@imperial.ac.uk}.
}
}
\author{
 Ruitong Huang \\
 Department of Computing Science\\
 University of Alberta, AB, Canada \\
 \texttt{ruitong@ualberta.ca}  \vspace{-0.04cm}
 \And 
 Tor Lattimore \\
 School of Informatics and Computing \\
 Indiana University, IN, USA \\
 \texttt{tor.lattimore@gmail.com} \vspace{-0.04cm}
 \And
 Andr\'as Gy\"orgy \\
 Dept. of Electrical \& Electronic Engineering\\ 
 Imperial College London, UK\\ 
 \texttt{a.gyorgy@imperial.ac.uk} 
 \And 
 Csaba Szepesv\'ari\\
 Department of Computing Science\\
 University of Alberta, AB, Canada \\
 \texttt{szepesva@ualberta.ca}
}

\author{Ruitong Huang \and Tor Lattimore \and Andr\'as Gy\"orgy \and Csaba Szepesv\'ari}

\begin{document}

\maketitle

\begin{abstract}
The follow the leader (FTL) algorithm, perhaps the simplest of all online learning algorithms,
is known to perform well when the loss functions it is used on are convex and positively curved.
In this paper we ask whether there are other ``lucky'' settings when FTL achieves sublinear, ``small'' regret.
In particular, we study the fundamental problem of linear prediction over a non-empty convex, compact domain.
Amongst other results, we prove that the curvature of  the boundary of the domain can act as if the losses
were curved: In this case, we prove that as long as 
the mean of the loss vectors have positive lengths bounded away from zero, 
FTL enjoys a logarithmic growth rate of regret, while, e.g., for polytope domains and stochastic data it enjoys
finite expected regret. 
Building on a previously known meta-algorithm, we also get
 an algorithm that simultaneously enjoys the worst-case guarantees and the bound available for FTL.
%while it approaches the smaller regret of FTL when the data is ``easy''.
\end{abstract}

\section{Introduction}
Learning theory traditionally has been studied in a statistical framework, discussed at length, for example, by \citet{SSS14:book}.
The issue with this approach is that the analysis of the performance of learning methods seems to critically depend
on whether the data generating mechanism satisfies some probabilistic assumptions. 
Realizing that these assumptions are not necessarily critical, much work has been devoted recently to 
studying learning algorithms in the so-called online learning framework \citep{CBLu06:book}. %SSS14:book}.
The online learning framework makes minimal assumptions about the data generating mechanism,
while allowing one to replicate results of the statistical framework through online-to-batch conversions
\citep{CBCoG04:OnlineToBatch}.
By following a minimax approach, however, results proven in the online learning setting, at least initially, led to rather
conservative results and algorithm designs, failing to capture how more regular, ``easier'' data, may give rise to
faster learning speed. This is problematic as it may suggest overly conservative learning strategies, missing
opportunities to extract more information when the data is nicer. Also, it is hard to argue that data resulting from
passive data collection, such as weather data, would ever be adversarially generated (though it is equally hard
to defend that such data satisfies precise stochastic assumptions).
Realizing this issue, during recent years much work has been devoted to understanding what regularities and how can lead to 
faster learning speed.
For example, much work has been devoted to showing that faster learning speed (smaller ``regret'') can be achieved
in the online convex optimization setting when the loss functions are ``curved'', such 
as when the loss functions are strongly convex or exp-concave, %and on how to adaptively exploit this property,
or when the losses show small variations, or the best prediction in hindsight has a small total loss, and that these properties can be exploited in an adaptive manner  (e.g.,
\citealt{MF92}, \citealt{FrSc97},
\citealt{gaivoronski2000stochastic},
\citealt{CBLu06:book},
\citealt{hazan2007logarithmic},
\citealt{bartlett2007adaptive},
\citealt{kakade2009mind},
\citealt{orabona2012beyond},
\citealt{RakhlinS13},
\citealt{vanerven2015fast},
\citealt{foster2015adaptive}).
%Huge literature on adaptive and fast rates, impossible to summarize here (and out of scope), but, e.g., 
%\citet{orabona2012beyond,vanerven2015fast,foster2015adaptive} review and unify many previous approaches.

In this paper we contribute to this growing literature by studying online linear prediction and the follow the leader (FTL) algorithm.
Online linear prediction is arguably the simplest yet fundamental of all the learning settings, and lies at the heart of online
convex optimization, while it also serves as an abstraction of core learning problems such as prediction with expert advice.
FTL, the online analogue of empirical risk minimization of statistical learning, is the simplest learning strategy, one can think of.
Although the linear setting of course removes the possibility of exploiting the curvature of losses, as we will see, there are
multiple ways online learning problems can present data that allows for small regret, even for FTL.
As is it well known, in the worst case,
FTL suffers a linear regret (e.g., Example 2.2 of \citet{SS12:Book}). 
However, for ``curved'' losses (e.g., exp-concave losses), FTL was shown to achieve small (logarithmic) regret
(see, e.g., \citet{MF92,CBLu06:book,gaivoronski2000stochastic,hazan2007logarithmic}).

In this paper we take a thorough look at FTL in the case when the losses are linear, 
but the problem perhaps exhibits other regularities.
The motivation comes from the simple observation that, for prediction over the simplex, when
the loss vectors are selected independently of each other from a distribution with a bounded support with a
nonzero mean, FTL quickly locks onto selecting the loss-minimizing vertex of the simplex, achieving finite expected regret.
In this case, FTL is arguably an excellent algorithm.
In fact, FTL is shown to be the minimax optimizer for the binary losses in the  stochastic expert setting in the paper of \citet{kotlowskiminimax}.
Thus, we ask the question of whether there are other regularities that allow FTL 
to achieve nontrivial performance guarantees.
Our main result shows that when the decision set (or constraint set) has a sufficiently ``curved'' boundary (equivalently, if it is strongly convex) and the linear loss is bounded away from $0$, FTL 
is able to achieve logarithmic regret even in the adversarial setting, thus opening up a new
way to prove fast rates not based on the curvature of losses, but on that of the boundary of the constraint set and non-singularity of the linear loss.
In a matching lower bound we show that this regret bound is essentially unimprovable.
%As we will show in an essentially matching lower bound, this regret bound must increase with the inverse
%of the norm of the mean of losses (i.e., small means makes the regret bound explode).
We also show an alternate bound for polytope constraint sets, which allows us to prove that 
(under certain technical conditions) for stochastic problems the expected regret of FTL will be finite.
To finish, we use ($\cA$, $\cB$)-prod of \citet{sani2014exploiting} to design an algorithm 
that adaptively interpolates between the worst case $O(\sqrt{n\log n})$ regret and the smaller regret bounds,
which we prove here for ``easy data.'' We also show that if the constraint set is the unit ball, both the follow the regularized leader (FTRL) algorithm and a combination of FTL and shrinkage, which we call follow the shrunken leader (FTSL), achieve logarithmic regret for easy data. Simulation results on artificial data complement the theoretical findings. 
%to illustrate the theory

While we believe that we are the first to point out that the curvature of the constraint set $\cW$ can help in speeding up learning,
this effect is known in convex optimization since at least the work of  \citet{LePo66},
who showed that exponential rates are attainable for strongly convex constraint sets if the norm of the gradients of the objective function admit a uniform lower bound. \todoc{Longer version: This is their Theorem 6.1, part (5).}
More recently, \citet{garber2014faster} %removes a restricting technical condition assumed by \citet{LePo66} and
proved an $O(1/n^2)$ optimization error bound (with problem-dependent constants) for the Frank-Wolfe algorithm for strongly convex and smooth objectives and strongly convex constraint sets.
The effect of the shape of the constraint set was also discussed by \citet{abbasi2010forced} who demonstrated $O(\sqrt{n})$ regret in the linear bandit setting.
While these results at a high level are similar to ours, our proof technique is rather different than that used there.
% that the curvature of the constraint set leads to faster convergence
%in the Frank-Wolfe algorithm
%for smooth, strongly convex functions, using Frank-Wolfe, $O(1/t^2)$ optimization error is attainable after $t$ iterations.
  
\todoc[inline]{
Interpolating between stochastic and adversarial settings: \citet{bubeck2012best}.
I think Rakhlin and Karthrik also write about this. What did they write? Cite them.}
\todor[inline]{\citep{abernethy2008optimal} section 4.2 talks about lower bound for linear game with constraint sets being balls. \citep{abernethy2009stochastic} relates the regret to the flatness of $\Phi$ and the Bregman divergence. \citep{abernethy2014online} Bregman divergence again.}
  
\todoc[inline]{
\citet{MF92}  considers the following assumption (dropping measurability and other technical requirements):
Let $\ell: \cF \times \cW \to [0,\infty)$ be a fixed loss function.
For a probability distribution $P$ over $\cF$, let $w^*(P) = \argmin_{w\in \cW} \int \ell(f,w) P(df)$.
Further, for $\alpha\in [0,1]$, $f\in \cF$, let $P_{\alpha,x} = P+ \alpha( \delta_f - P)$, where $\delta_f$ is the Dirac measure that puts all the weight to $f$. (Note that $P_{\alpha,x}-P = \alpha (\delta_f-P)$.)
Then, the assumption is that for some $L>0$ and for all $f\in \cF$,
 $|\ell(f, b^*(P) ) - \ell( f, b^*(P_{\alpha,f}))| \le \alpha L$ (a form of a Lipschitz condition).
Under this assumption they show that FTL achieves logarithmic regret.
How does this assumption relate to our smoothness assumption?
}
\todor[inline]{More about stability. \citep{saha2012interplay}. Such stability is usually achieved by the strongly convexity of the loss function.}

%\vspace{-0.05cm}
\section{Preliminaries, online learning and the follow the leader algorithm}
\label{sec:notation}
% \begin{wrapfigure}{r}{6cm}
%	\vspace{-.5cm}
%	\centering
%	\begin{algorithmic}[1]
%		\FOR{$t = 1 \text{ to } n$}
%		\STATE Learner predicts $w_t\in \cW$;
%		\STATE Environment picks $\ell_t\in \cL$;
%		\STATE Learner suffers $\ell_t(w_t)$ and learns $\ell_t$.
%		\ENDFOR
%	\end{algorithmic}
%	\caption{Online Learning} 
%  	\label{fig:onlinelearning}
%	\vspace{-.5cm}
%\end{wrapfigure} 
We consider the standard framework of online convex optimization, where a learner and an environment interact in a sequential manner in $n$ rounds: In round every round $t=1,\ldots,n$, first the learner predicts $w_t\in \cW$. Then the environment picks a loss function $\ell_t\in \cL$, and the learner suffers loss $\ell_t(w_t)$ and observes $\ell_t$. 
%as shown in \cref{fig:onlinelearning}.
%as follows \citep{Zin03}:
Here, $\cW$ is a non-empty, compact \todoa{Changed to compact here, as we need it anyways for the linear losses} convex subset of $\real^d$ and
 $\cL$ is a set of convex functions, mapping $\cW$ to the reals.
 The elements of $\cL$ are called loss functions.
The performance of the learner is measured in terms of its regret,
\[
R_n = \sum_{t=1}^n \lt(w_t) - \min_{w\in \cW}\sum_{t=1}^n \lt(w)\,.
\]
 
The simplest possible case, which will be the focus of this paper,
is when the losses are linear, i.e., when $\lt(w) = \ip{f_t,w}$ for some $f_t\in \cF\subset \real^d$.
\newcommand{\tlt}{\tilde{\ell}_t}
In fact, the linear case is not only simple, but is also fundamental since the case of nonlinear loss functions can be reduced to it: Indeed, even if the losses are nonlinear, 
defining $f_t \in \partial \lt(w_t)$ to be a subgradient%
\footnote{
We let $\partial g(x)$ denote the subdifferential of a convex function $g:\dom(g) \to \R$ at $x$,
i.e., $\partial g(x) = \set{\theta\in \R^d}{g(x') \ge g(x) + \ip{\theta, x'-x} \,\, \forall x'\in \dom(g) }$,
where $\dom(g)\subset \R^d$ is the domain of $g$.
} 
of $\lt$ at $w_t$ and  letting $\tlt(u) = \ip{f_t,u}$, by the definition of subgradients,
$\lt(w_t)-\lt(u) \le \lt(w_t)-(\lt(w_t)+\ip{f_t,u-w_t}) = \tlt(w_t)-\tlt(u)$, hence for any $u\in \cW$,
\[
\sum_t \lt(w_t) - \sum_t \lt(u) \le \sum_t \tilde{\lt}(w_t) - \sum_t \tilde{\lt}(u)\,.
\]
In particular, if an algorithm keeps the regret small no matter how the linear losses are selected
(even when allowing the environment to pick losses based on the choices of the learner), 
the algorithm can also be used to keep the regret small in the nonlinear case. 
%showing that, in this sense, the linear case is fundamental.
Hence, in what follows we will study the linear case $\ell_t(w)=\ip{f_t,w}$ and, in particular, we will study the regret
of the so-called ``Follow The Leader'' (FTL) learner, which, in round $t\ge 2$ 
picks
% $w_t$ such that $w_t$ minimizes the 
%total loss $\sum_{i=1}^{t-1} \ell_i(w)$ accumulated so far, cf. \cref{fig:ftl}
\begin{align*}
w_t = \argmin_{w\in \cW} \sum_{i=1}^{t-1} \ell_i(w)\,.
\end{align*}
For the first round, $w_1\in \cW$ is picked in an arbitrary manner.
When $\cW$ is compact, the optimal $w$ of $\min_{w\in\cW} \sum_{i=1}^{t-1}\inpro{w}{f_t}$ is attainable,
which we will assume henceforth.
If multiple minimizers exist, we simply fix one of them as $w_t$.
We will also assume that $\cF$ is non-empty, compact and convex. \todoc{Why compact? Convex? How is this used?}\todoa{It is used with $\Phi$ and its Bregman divergence. Could be relaxed but this is the standard way.}
\if0
 \begin{wrapfigure}{R}{5cm}
	\vspace{-.5cm}
	\centering
	\begin{algorithmic}
		\STATE In round $t\ge 2$, predict:
		\begin{align*}
		w_t = \argmin_{w\in \cW} \sum_{i=1}^{t-1} \ell_i(w)\,.
		\end{align*}
		while in round one predict arbitrarily.
	\end{algorithmic}
	\caption{Follow the Leader (FTL)} 
  	\label{fig:ftl}
	\vspace{-.5cm}
\end{wrapfigure} 
\fi

\subsection{Support functions}
Let $\Theta_t = -\frac1t \sum_{i=1}^t f_i$ be the negative average of the first $t$ vectors
in $(f_t)_{t=1}^n$, $f_t\in \cF$.
For convenience, we define $\Theta_0 := 0$.
Thus, for $t\ge 2$,
\begin{align*}
w_t =  \argmin_{w\in\cW} \sum_{i=1}^{t-1} \ip{ w, f_i } = \argmin_{w\in\cW} \ip{ w, -\Theta_{t-1} }
= \argmax_{w\in \cW} \ip{w,\Theta_{t-1}}\,.
\end{align*}
Denote by $\Phi(\Theta) = \max_{w\in\cW} \langle w, \Theta\rangle$ the so-called \emph{support function} of $\cW$. 
The support function, being the maximum of linear and hence convex functions, is itself convex.
Further $\Phi$ is positive homogenous: for $a\ge 0$ and $\theta\in \R^d$, $\Phi(a \theta) = a\Phi(\theta)$.
It follows then that the epigraph $\epi(\Phi) = \set{ (\theta,z)}{ z\ge \Phi(\theta), z\in \R, \theta\in \R^d }$ of $\Phi$ is a cone,
since for any $(\theta,z)\in \epi(\Phi)$ and $a\ge 0$, 
$az \ge a \Phi(\theta) = \Phi(a\theta)$, $(a\theta,az)\in \epi(\Phi)$ also holds.

%%%%%%%%%%%%%%%%%%%%%%%%%%%%%%%%%%%%%%%%%%%%%%%%%%%%
The differentiability of the support function is closely tied to whether in the FTL algorithm the choice of $w_t$ is
uniquely determined:
\begin{prop} 
\label{prop:derivativePhi}
Let $\cW\ne \emptyset$ be convex and closed.
Fix $\Theta$ and let $\cZ:= \set{w\in \cW}{\inpro{w}{\Theta} =  \Phi(\Theta) }$.
Then, $\partial \Phi(\Theta) = \cZ$ and, in particular,
$\Phi(\Theta)$ is differentiable at $\Theta$ if and only if 
$\max_{w\in\cW} \inpro{w}{\Theta}$ has a unique optimizer.
In this case, $\nabla \Phi(\Theta) = \argmax_{w\in \cW} \ip{w,\Theta}$.
\end{prop}
The proposition follows from Danskin's theorem when $\cW$ is compact
(e.g., Proposition B.25 of \citealt{bertsekas99nonlinear}), 
but a simple
direct argument can also be used to show that it also remains true even when $\cW$ is unbounded. \todor{Remove the footnote.}
\footnote{
The proofs not given in the main text can be found in the appendix.
} 
By \cref{prop:derivativePhi},
when $\Phi$ is differentiable at $\Theta_{t-1}$,
$w_t = \nabla \Phi(\Theta_{t-1})$.
%\todoa{This part is somewhat problematic, as $\phi(\Theta)$ is not defined for some $Theta$ if 

\section{Non-stochastic analysis of FTL}
\label{sec:FTL}
We start by rewriting the regret of FTL in an equivalent form, which shows that we can expect FTL to enjoy a small
regret when successive weight vectors move little. 
A noteworthy feature of the next proposition is that rather than bounding the regret from above, it gives an equivalent
expression for it. 
\begin{prop}
\label{prop:regretabel}
The regret $R_n$ of FTL satisfies
\begin{align*}
R_n & =  \sum_{t=1}^n t\,\ip{ w_{t+1}-w_t,\Theta_t}  \,.
\end{align*}
\end{prop}
The result is a direct corollary of Lemma 9 of \citet{McMahan10:Equiv}, which holds 
for any sequence of losses, even in the lack of convexity.
It is also a tightening of the well-known inequality $R_n \le \sum_{t=1}^n \ell_t(w_t)-\ell_t(w_{t+1})$,
which again holds for arbitrary loss sequences (e.g., Lemma 2.1 of \citet{SS12:Book}).
To keep the paper self-contained, we give an elegant, short direct proof, based on the summation by parts formula:
\begin{proof}
The summation by parts formula states that for any $u_1,v_1,\dots,u_{n+1},v_{n+1}$ reals,
$
\sum_{t=1}^n u_t\,(v_{t+1}-v_t) = (u_{t+1}v_{t+1}-u_1 v_1) - \sum_{t=1}^n (u_{t+1}-u_t)\,v_{t+1} 
$.
Applying this to the definition of regret
with $u_t:=w_{t,\cdot}$ and $v_{t+1} := t\Theta_{t}$, we get
\begin{align*}
R_n 
& = -\sum_{t=1}^n \ip{w_t,t\Theta_t - (t-1)\Theta_{t-1}} + \ip{w_{n+1},n\Theta_n}   \\
& = - \left\{ 
		\bcancel{\ip{w_{n+1},n\Theta_n}} - 0 - \sum_{t=1}^n \ip{w_{t+1}-w_t,t\Theta_t}  \right\} +
		\bcancel{\ip{w_{n+1},n\Theta_n}}.
\end{align*}
\end{proof}
Our next proposition gives another formula that is equal to the regret.
As opposed to the previous result, this formula is appealing as
it is independent of $w_t$; but it directly connects the sequence $(\Theta_t)_t$ to the 
geometric properties of $\cW$ through the support function $\Phi$.
For this proposition we will momentarily assume that $\Phi$ is differentiable at $(\Theta_t)_{t\ge 1}$;
a more general statement will follow later.
%To state the proposition, recall that the Bregman divergence from $u$ to $v$ induced by a convex function $g$ 
%which is differentiable at $u$ is $D_g(v,u) = g(v) - g(u) - \ip{\nabla g(u), v-u}$.
%When $\Phi$ is differentiable at $\Theta_{t-1} (\neq0)$, $w_t = \nabla \Phi(\Theta_{t-1})$.

%The next proposition shows that the regret of FTL is in fact tied to the smoothness of the function $\Phi$.
%%%%%%%%%%%%%%%%%%%%%%%%%%%%%%%%%%%%%%%%%%%%%%%%%%%%
\begin{prop} 
\label{prop:R_nBregmanDivergence}
If $\Phi$ is differentiable at $\Theta_1, \ldots, \Theta_n$,  %\todoc{There is an issue that $\Phi$ is not differentiable at $\Theta_0 = 0$.}
\begin{align}
\label{eq:regreteq}
R_n = \sum_{t=1}^{n} t\,D_{\Phi}(\Theta_t,\Theta_{t-1})\,,
\end{align}
where $D_{\Phi}(\theta', \theta) = \Phi(\theta') - \Phi(\theta) - \ip{ \nabla\Phi(\theta), \theta' - \theta}$ is the Bregman divergence of $\Phi$
and
we use the convention that $\nabla\Phi(0) = w_1$.
\end{prop}
%%%%%%%%%%%%%%%%%%%%%%%%%%%%%%%%%%%%%%%%%%%%%%%%%%%%
\begin{proof}
Let $v = \argmax_{w\in\cW}\inpro{w}{\theta}$, 
$v' = \argmax_{w\in \cW}\ip{w,\theta'}$.
When $\Phi$ is differentiable at $\theta$,
\begin{align}
D_{\Phi}(\theta', \theta) & = \Phi(\theta') - \Phi(\theta) - \inpro{\nabla\Phi(\theta)}{\theta' \!- \theta} 
  =  \inpro{v'}{\theta'} \!- \inpro{v}{\theta} -\inpro{v}{\theta' \!- \theta} = \inpro{v'\!-v}{\theta'}\,. 
\label{eq:bregman}
\end{align}
Therefore, by \cref{prop:regretabel}, $R_n = \sum_{t=1}^{n} t\ip{ w_{t+1}-w_t,\Theta_t} = \sum_{t=1}^{n} t\,D_{\Phi}(\Theta_t,\Theta_{t-1})$.
\end{proof}
When $\Phi$ is non-differentiable at some of the points $\Theta_1,\dots,\Theta_n$, the equality in the above proposition can be replaced with inequalities.
Defining the upper Bregman divergence %s for $\Phi$:
$\uD_{\Phi}(\theta', \theta) 
= \sup_{w\in \partial \Phi(\theta)} \Phi(\theta') - \Phi(\theta) - \ip{ w, \theta' - \theta}$ and the lower Bregman divergence $\lD_{\Phi}(\theta', \theta)$ similarly with $\inf$ instead of $\sup$, 
%
%\begin{align*}
%\uD_{\Phi}(\theta', \theta) 
%& = \sup_{w\in \partial \Phi(\theta)} \Phi(\theta') - \Phi(\theta) - \ip{ w, \theta' - \theta}\,,\\
%\lD_{\Phi}(\theta', \theta) 
%& = \inf_{w\in \partial \Phi(\theta)} \Phi(\theta') - \Phi(\theta) - \ip{ w, \theta' - \theta}\,.
%\end{align*}
we can easily obtain an analogue of \Cref{prop:R_nBregmanDivergence}:
%These definitions give rise to
%a trivial extension of \eqref{eq:bregman}:
%$\lD_{\Phi}(\theta', \theta) \le \ip{v'-v,\theta'} \le \uD_{\Phi}(\theta', \theta)$. Combining with  \cref{prop:regretabel} as above, we obtain the following result:
%The following holds then:
%%%%%%%%%%%%%%%%%%%%%%%%%%%%%%%%%%%%%%%%%%%%%%%%%%%%
%\begin{prop} 
%\label{prop:R_nBregmanDivergence2}
%We have
\begin{align}
\label{eq:regreteq_alt}
\sum_{t=1}^{n} t\,\lD_{\Phi}(\Theta_t,\Theta_{t-1})
\le
R_n 
\le \sum_{t=1}^{n} t\,\uD_{\Phi}(\Theta_t,\Theta_{t-1})\,.
\end{align}
%\end{prop}
%%%%%%%%%%%%%%%%%%%%%%%%%%%%%%%%%%%%%%%%%%%%%%%%%%%%
%The proof follows the same steps as the proof of the previous proposition, except that instead of \eqref{eq:bregman}
%we use
%when $w_t$ was replaced by $\nabla \Phi(\Theta_{t-1})$, now
%we just  keep $w_t$ up to the end, when we take the infimum (supremum)
%over all elements of $\partial \Phi(\Theta_{t-1})$ to get the lower (respectively, upper)
%bound, exploiting that $w_t \in \partial \Phi(\Theta_{t-1})$.

\subsection{Constraint sets with positive curvature}
The previous results show in an implicit fashion that the curvature of $\cW$ controls the regret. Before presenting our first main result, which makes this connection explicit, we define some basic notions from differential geometry related to the curvature (all differential geometry concept and results that we need can be found in Section 2.5 of \citealp{Sch14:ConvexBodies}).

%\subsubsection{Curvature of planar curves and convex bodies}
Given a $C^2$ (twice continuously differentiable) planar curve $\gamma$ in $\real^2$, there exists a parametrization with respect to the curve length $s$, such that $\|\gamma'(s)\| = \|\left(x'(s), y'(s)\right)\| = x'(s)^2 + y'(s)^2=1$. Under the curve length parametrization, the curvature of $\gamma$ at $\gamma(s)$ is $\|\gamma''(s)\|$.
Define the unit normal vector $\bf{n}(s)$ as the unit vector that is perpendicular to $\gamma'(s)$.\footnote{There exist two unit vectors that are perpendicular to $\gamma'(s)$ for each point on $\gamma$. Pick the ones that are consistently oriented.}
Note that $\bf{n}(s)\cdot \gamma'(s) = 0$. Thus $0=\left(\bf{n}(s)\cdot \gamma'(s)\right)' = \bf{n}'(s)\cdot\gamma'(s) + \bf{n}(s)\cdot \gamma''(s)$, and $\|\gamma''(s)\| = \|\bf{n}(s)\cdot \gamma''(s)\| = \|\bf{n}'(s)\cdot\gamma'(s)\| = \|\bf{n}'(s)\|$. Therefore, the curvature of $\gamma$ at point $\gamma(s)$ is the length of the differential of its unit normal vector.
 
%{\bf Principal Curvature}\\
Denote the boundary of $\cW$ by $\bd(\cW)$.
We shall assume that $\cW$ is $C^2$, that is, $\bd(\cW)$ is a twice continuously differentiable submanifold
of $\R^d$. 
We denote the tangent plane of $\bd(\cW)$ at point $w$ by $T_w\cW$. Now there exists a unique unit vector at $w$ that is perpendicular to $T_w\cW$ and points outward of $\cW$.
In fact, one can define a continously differentiable normal unit vector field on $\bd(\cW)$, $u_{\cW}: \bd(\cW) \to \bS^{d-1}$, the so-called Gauss map, which maps a boundary point $w\in \bd(\cW)$ to the unique outer normal vector to $\cW$ at $w$, where
$\bS^{d-1}=\set{x\in\R^d}{\|x\|_2=1}$ denotes the unit sphere in $d$-dimensions. 
The differential of the Gauss map, $\nabla u_{\cW}(w)$, defines a linear endomorphism of $T_w\cW$. Moreover, $\nabla u_{\cW}(w)$ is a self-adjoint operator, with nonnegative eigenvalues.
The differential of the Gauss map, $\nabla u_{\cW}(w)$, describes the curvature of $\bd(\cW)$ via the second fundamental form. In particular, the \emph{principal curvatures} of $\bd(\cW)$ at $w\in\bd(\cW)$ is defined as the eigenvalues of $\nabla u_{\cW}(w)$.   
Perhaps a more intuitive, yet equivalent definition, is that the principal curvatures are the eigenvalues
of the Hessian of $f=f_w$ in the parameterization $t\mapsto w+t-f_w(t) u_{\cW}(w)$ of $\bd(\cW)$
which is valid in a small open neighborhood of $w$, where $f_w: T_w \cW \to [0,\infty)$ is
a suitable convex, nonnegative valued function that also satisfies $f_w(0)= 0$ and where $T_w \cW$, 
a hyperplane of $\R^d$,
denotes the tangent space of $\cW$ at $w$, 
obtained by taking the support plane $H$ of $\cW$ at $w$ and shifting it by $-w$.
Thus, the principal curvatures at some point $w\in \bd(\cW)$ describe the local shape of $\bd(\cW)$ 
up to the second order. 
In this paper, we are interested in the minimum principal curvature at $w\in\bd(\cW)$, which can be intepreated as the minimum curvature at $w$ over all the planar curves $\gamma \in \bd(\cW)$ that go through $w$.

A related concept that has been used in convex optimization to show fast rates is that of a strongly convex constraint set \citep{LePo66,garber2014faster}:
$\cW$ is $\lambda$-strongly convex with respect to the norm $\norm{\cdot}$ if, for  any $x,y\in \cW$ and $\gamma\in [0,1]$, the $\norm{\cdot}$-ball with origin $\gamma x + (1-\gamma) y$ and radius $\gamma(1-\gamma) \lambda \norm{x-y}^2/2 $ is included in $ \cW$. %\footnote{$B(o,r)$ denotes the $\norm{\cdot}$-ball with origin $o$ and radius $r$.}
We show in \cref{strongconvex} in the appendix that a convex body
 $\cW \in C^2$ is $\lambda$-strongly convex with respect to $\norm{\cdot}_2$ if and only if the principal curvatures of the surface $\bd(\cW)$ are all at least $\lambda$.

As promised, our next result connects the principal curvatures of $\bd(\cW)$ to the regret of FTL
and shows that FTL enjoys logarithmic regret for highly curved surfaces, as long as $\norm{\Theta_t}_2$ 
is bounded away from zero.
\begin{thm}
\label{thm:R_curvesurface}
Let $\cW\subset \R^d$ be a $C^2$ convex body\footnote{Following \citet{Sch14:ConvexBodies}, a convex body of $\R^d$ is any
 non-empty, compact, convex subset of $\R^d$.}  with $d\ge 2$.
Let $M = \max_{f\in \cF} \norm{f}_2$ and assume that $\Phi$ is differentiable at $(\Theta_t)_{t}$.
Assume that the principal curvatures of the surface $\bd(\cW)$ 
are all at least $\lambda_0$ for some constant $\lambda_0>0$ and $L_n:=\min_{1\le t \le n} \|\Theta_t\|_2 >0$. 
Choose $w_1\in \bd(\cW)$.
Then
\[
R_n \le \frac{2M^2}{\lambda_0 L_n}(1+ \log(n))\,.
\]
\end{thm}

As we will show later in an essentially matching lower bound, this bound is tight, showing that the forte  %%%fort\'e???
of FTL is when $L_n$ %:=\min_{1\le t\le n} \norm{\Theta_t}_2$ 
is bounded away from zero and $\lambda_0$ is large.
Note that the bound is vacuous as soon as $L_n =O( \log(n)/n )$
and is worse than the minimax bound of $O(\sqrt{n})$ when $L _n = o( \log(n)/\sqrt{n} )$. 
One possibility to reduce the bound's sensitivity to $L_n$  
is to use the trivial bound $\ip{w_{t+1}-w_t,\Theta_t} \le L W = L \sup_{w,w'\in \cW} \norm{w-w'}_2$ 
for indices $t$ when $\norm{\Theta_t}\le L$. Then, by optimizing the bound over $L$, 
one gets a data-dependent bound
of the form $\inf_{L>0} \left(\frac{2M^2}{\lambda_0 L} (1+\log(n)) +  LW \, \sum_{t=1}^n t \,\one{ \norm{\Theta_t}\le L }\right)$,
which is more complex, but is free of $L_n$ and thus reflects the nature of FTL better.
Note that in the case of stochastic problems, where $f_1,\ldots,f_n$ are independent and identically distributed (i.i.d.) with $\mu := -\Exp{\Theta_t}\ne 0$, the probability that $\norm{\Theta_t}_2 < \norm{\mu}_2/2$ is exponentially small in $t$. Thus, selecting $L=\norm{\mu}_2/2$ in the previous bound, the contribution of the expectation of the second term is $O(\norm{\mu}_2W)$, giving an overall bound of the form $O(\frac{M^2}{\lambda_0 \norm{\mu}_2}\log(n)+\norm{\mu}_2 W)$. \todor{the second term should be $\frac{W}{\|\mu\|_2^3}$. The sum of the probabilities brings a term $\frac{1}{\|\mu\|_2^4}$.} \todor{I correct the sum to $\min$.}\todoa{The correct is the sum. And there is no need to multiply the second term by $n$, since it is a sum of exponentially decaying sequence}
After the proof we will provide some simple examples that should make it more intuitive 
how the curvature of $\cW$ helps keeping the regret of FTL small.
\begin{proof}
Fix  $\theta_1, \theta_2 \in \real^d$ and let 
$w^{(1)} = \argmax_{w\in\cW}\inpro{w}{\theta_1}$,
$w^{(2)} = \argmax_{w\in\cW}\inpro{w}{\theta_2}$. 
Note that if $\theta_1,\theta_2\ne 0$ then $w^{(1)} , w^{(2)}  \in \bd(\cW)$. 
Below we will show that
\begin{align*}
\inpro{w^{(1)} - w^{(2)} }{\theta_1} 
	& \le \frac{1}{2\lambda_0} \frac{\|\theta_2 - \theta_1\|_2^2}{\|\theta_2\|_2}\,.
	 \numberthis\label{eq:middletheta}
\end{align*}
\cref{prop:regretabel}  suggests that it suffices to bound $\ip{w_{t+1}-w_t,\Theta_t}$. By \eqref{eq:middletheta}, we see
 that it suffices to bound how much $\Theta_t$ moves.
A straightforward calculation shows that $\Theta_t$ cannot move much:
%\begin{lemma}
for any norm $\norm{\cdot}$ on $\cF$, we have 
\begin{align}
\|\Theta_t - \Theta_{t-1} \| & = \left\|\frac{1}{t-1}\sum_{i=1}^{t-1} f_i - \frac{1}{t}\sum_{i=1}^{t} f_i \right\| 
	 = \left\| \sum_{i=1}^{t-1} \left( \frac{1}{t-1} - \frac{1}{t}\right) f_i- \frac{1}{t}f_t\right\| \nonumber \\
	& \le \left\| \sum_{i=1}^{t-1} \left( \frac{1}{t-1} - \frac{1}{t}\right) f_i \right\| + \left\| \frac{1}{t}f_t\right\| 
	 = \left\| \sum_{i=1}^{t-1} \frac{1}{t(t-1)} f_i \right\| + \left\| \frac{1}{t}f_t\right\| \nonumber \\
	 & = \frac{1}{t} \left\| \frac{1}{t-1} \sum_{i=1}^{t-1} f_i\right\| + \frac{1}{t}\left\|f_t\right\| 
	 \le \frac{2}{t}M\,. \label{prop:avgdiff}
\end{align}
where $M = \max_{f\in\cF} \|f\|$ is a constant that depends on $\cF$ and the norm $\norm{\cdot}$.
%\end{lemma}

Combining inequality \eqref{eq:middletheta} with \cref{prop:regretabel} and \eqref{prop:avgdiff}, we get
\begin{align*}
R_n &= \sum_{t=1}^{n} t\ip{ w_{t+1}-w_t,\Theta_t} %\\ & 
\le \sum_{t=1}^{n} \frac{t}{2\lambda_0} \frac{\|\Theta_t - \Theta_{t-1}\|_2^2}{\|\Theta_{t-1}\|_2} \\
&\le \frac{2M^2}{\lambda_0}\sum_{t=1}^{n} \frac{1}{t\|\Theta_{t-1}\|_2} \le \frac{2M^2}{\lambda_0L_n} \sum_{t=1}^{n} \frac{1}{t}
\le \frac{2M^2}{\lambda_0L_n} (1+\log(n))\,.
\end{align*}
To finish the proof, it thus remains to show~\eqref{eq:middletheta}.

The following elementary lemma relates the cosine of the angle between two vectors $\theta_1$ and $\theta_2$ to the squared normalized distance between 
the two vectors, thereby reducing our problem to bounding the cosine of this angle.
For brevity, we denote by $\cos\inangle{\theta_1}{\theta_2}$
the cosine of the angle between $\theta_1$ and $\theta_2$. 
\begin{lemma}
\label{lem:upperboundcos}
For any non-zero vectors $\theta_1, \theta_2 \in \real^d$,
\begin{align}
1- \cos \inangle{\theta_1}{\theta_2} \le \frac{1}{2} \frac{\|\theta_1 - \theta_2\|_2^2}{\|\theta_1\|_2\|\theta_2\|_2}.
\label{eq:angleineq}
\end{align}
\end{lemma}
\begin{proof}
Note that $\|\theta_1\|_2\|\theta_2\|_2\cos\inangle{\theta_1}{\theta_2} = \inpro{\theta_1}{\theta_2}$.
Therefore, \eqref{eq:angleineq} is equivalent to 
$ 2\|\theta_1\|_2\|\theta_2\|_2 - 2\inpro{\theta_1}{\theta_2} \le \|\theta_1 - \theta_2\|_2^2 $,
which, by algebraic manipulations, is itself equivalent to $0 \le (\|\theta_1\|_2-\|\theta_2\|_2)^2$.
\end{proof}

\begin{figure}[h]
  \centering
	\includegraphics[height=4cm]{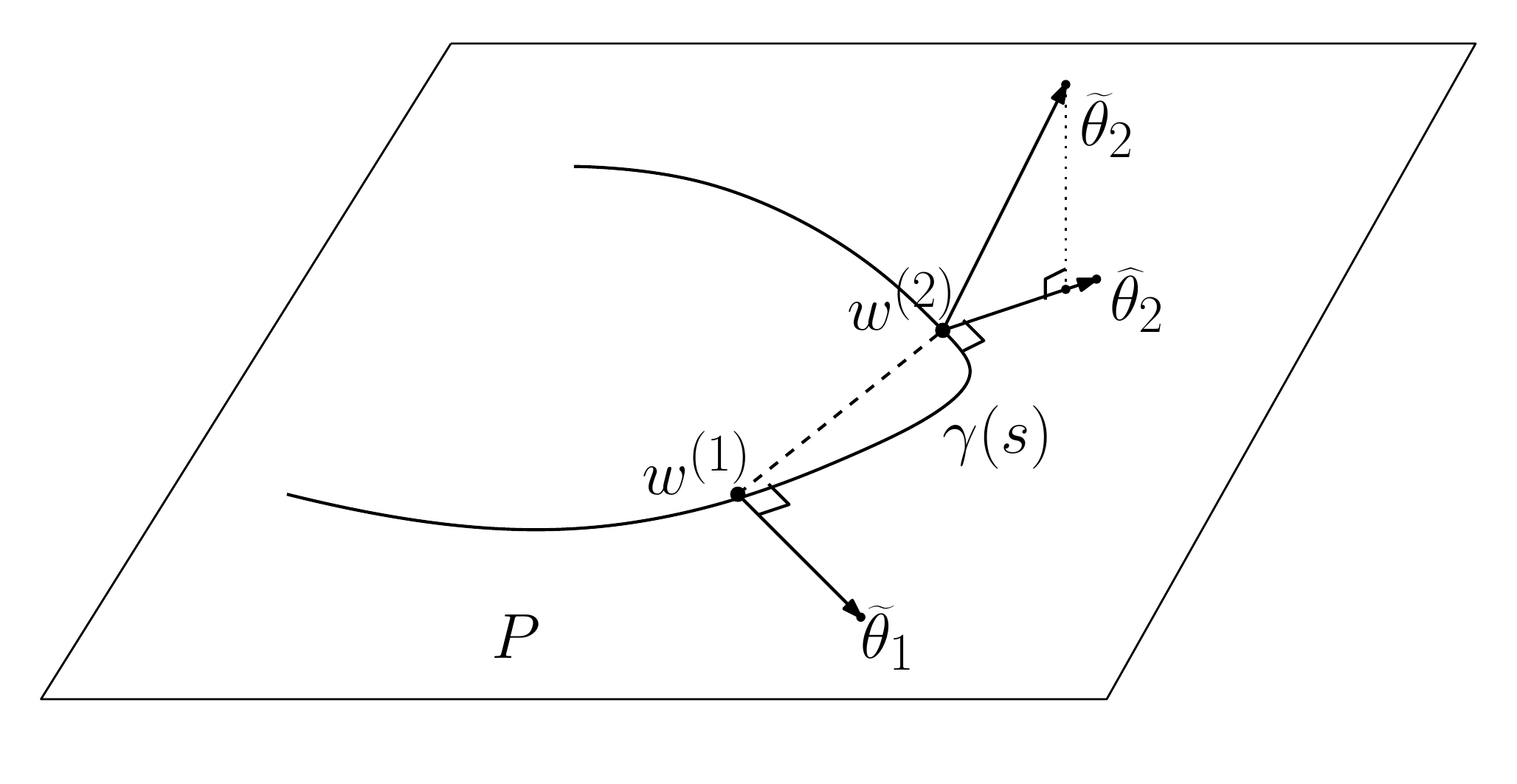}
  \caption{Illustration of the construction used in the proof of~\eqref{eq:middletheta}.} 
  \label{fig:cuttingplane}
\end{figure}

With this result, we see that it suffices to upper bound $\cos \inangle{\theta_1}{\theta_2}$ by $1-\lambda_0 \inpro{w^{(1)}-w^{(2)}}{\frac{\theta_1}{\|\theta_1\|_2}}$.
To develop this bound, let $\ttheta_i = \frac{\theta_i}{\|\theta_i\|_2}$ for $i=1,2$.
The angle between $\theta_1$ and $\theta_2$ is the same as the angle between 
the normalized vectors $\ttheta_1$ and $\ttheta_2$.
To calculate the cosine of the angle between $\ttheta_1$ and $\ttheta_2$,
let $P$ be a plane spanned by $\ttheta_1$ and $w^{(1)}-w^{(2)}$ and passing through $w^{(1)}$
($P$ is uniquely determined if $\ttheta_1$ is not parallel to $w^{(1)}-w^{(2)}$;
if there are multiple planes, just pick any of them). 
Further, let $\htheta_2\in \bS^{d-1}$ be the unit vector along the projection of $\ttheta_2$ onto the plane $P$, as indicated in \cref{fig:cuttingplane}.
Clearly, $\cos \inangle{\ttheta_1}{\ttheta_2} \le \cos \inangle{\ttheta_1}{\htheta_2}$.

Consider a curve $\gamma(s)$ on $\bd(\cW)$ connecting $w^{(1)}$ and $w^{(2)}$ that is defined by the intersection of $\bd(\cW)$ and $P$ and is parametrized by its curve length $s$ so that $\gamma(0) = w^{(1)}$ and $\gamma(l) = w^{(2)}$,
where $l$ is the length of the curve $\gamma$ between $w^{(1)}$ and $w^{(2)}$.
Let $u_{\cW}(w)$ denote the outer normal vector to $\cW$ at $w$ as before,
and let $u_\gamma\, : \, [0,l]\rightarrow \bS^{d-1}$ be such that $u_\gamma(s) = \htheta$ where $\htheta$ is the unit vector parallel to the projection of $u_{\cW}(\gamma(s))$ on the plane $P$. 
By definition, $u_\gamma(0) = \ttheta_1$ and $u_\gamma(l) = \htheta_2$.
Note that in fact $\gamma$ exists in two versions since $\cW$ is a compact convex body,
hence the intersection of $P$ and $\bd(\cW)$ is a closed curve.
Of these two versions we choose the one that satisfies that $\ip{\gamma'(s),\ttheta_1}\le 0$ for $s\in [0,l]$.\footnote{$\gamma'$ and $u'_\gamma$ denote the derivatives of $\gamma$ and $u$, respectively, which exist since $\cW$ is $C^2$.}
Given the above, we have
\begin{align*}
\cos \inangle{\ttheta_1}{\htheta_2} & = \inpro{\htheta_2}{\ttheta_1} 
	 = 1 \! + \inpro{\htheta_2 - \ttheta_1}{\ttheta_1} 
	  = 1\!+ \Big\langle\int_{0}^{l} u_\gamma'(s)\,\text{d}s, \ttheta_1 \Big\rangle
	 = 1\!+ \!\int_{0}^{l} \inpro{u_\gamma'(s)}{\ttheta_1} \,\text{d}s. \numberthis \label{eq:cosint}
\end{align*}
Note that $\gamma$ is a planar curve on $\bd(\cW)$, 
thus its curvature $\lambda(s)$ satisfies $\lambda(s) \ge \lambda_0$ for $s\in [0,l]$.
Also, for any $w$ on the curve $\gamma$, $\gamma'(s)$ is a unit vector parallel to $P$. 
Moreover, $u_\gamma'(s)$ is parallel to $\gamma'(s)$ and $\lambda(s) = \|u_\gamma'(s)\|_2$.
Therefore, 
\[
\inpro{u_\gamma'(s)}{\ttheta_1} = \|u_\gamma'(s)\|_2\inpro{\gamma'(s)}{\ttheta_1} \le \lambda_0\inpro{\gamma'(s)}{\ttheta_1},
\]
where the last inequality holds because $\inpro{\gamma'(s)}{\ttheta_1} \le 0$.
Plugging this into~\eqref{eq:cosint}, we get the desired
\begin{align*}
\cos \inangle{\ttheta_1}{\htheta_2}   
	& \le 1+ \lambda_0\, \int_{0}^{l} \, \inpro{\gamma'(s)}{\ttheta_1}  \,\text{d}s 
	    = 1+ \lambda_0 \Big\langle\int_{0}^{l} \gamma'(s) \,\text{d}s, \ttheta_1 \Big\rangle
	  = 1 - \lambda_0 \inpro{w^{(1)}  - w^{(2)}}{\ttheta_1}\,.
\end{align*}
Reordering and combining with~\eqref{eq:angleineq} we obtain
\begin{align*}
\inpro{w^{(1)}  - w^{(2)}}{\ttheta_1} 
	& \le \frac{1}{\lambda_0} \left( 1- \cos \inangle{\ttheta_1}{\htheta_2} \right)
	%									  \left( 1- \cos \inpro{\ttheta_1}{\ttheta_2} \right) 
	 \le \frac{1}{\lambda_0} \left( 1- \cos \inangle{\theta_1}{\theta_2} \right) 
	 \le \frac{1}{2\lambda_0} \frac{\|\theta_1 - \theta_2\|_2^2}{\|\theta_1\|_2\|\theta_2\|_2}\,.
	% \numberthis\label{eq:middletheta},
\end{align*}
Multiplying both sides by $\norm{\theta_1}_2$ gives~\eqref{eq:middletheta}, thus, finishing the proof.
\end{proof}

\begin{example}
	\label{ex:curvature}
The smallest principal curvature of some common convex bodies are as follows:
\begin{itemize}\setlength{\itemsep}{0pt}
\item The smallest principal curvature $\lambda_0$ of the Euclidean ball $\cW = \set{w}{\|w\|_2\le r}$ of radius $r$ 
satisfies $\lambda_0=\frac{1}{r}$.
\item Let $Q$ be a positive definite matrix.
If $\cW = \set{w}{w^\top Q w\le 1 }$ then $\lambda_0=\lambda_{\min}/\sqrt{\lambda_{\max}}$, 
where $\lambda_{\min}$ and $\lambda_{\max}$ are the minimal, respectively, maximal eigenvalues of $Q$.
(\citealt{Pol96} also derived this result for the strong convexity definition \eqref{sc:l2} in \cref{strongconvex}.)
\item
In general, let $\phi:\R^d \to \R$ be a $C^2$ convex function.
Then, for $\cW = \set{w}{\phi(w)\le 1}$, 
$\lambda_0=\min_{w\in\bd(\cW)}\min_{v\,:\,\|v\|_2=1, v\perp \phi'(w) }\frac{v^{\top}\nabla^2\phi(w) v}{\|\phi'(w)\|_2}$~. \todoa{We should prove this one}
\end{itemize}
\end{example}

We only prove the last statement, since it implies the other two.
\begin{proof}
Fix $w\in\bd(\cW)$. Note that $\phi'(w)$ is a normal vector at $w$ for $\bd(\cW)$, thus $T_w\cW = \seto{v: v\perp \phi'(w)}$.
Then the Gauss map $u_{\cW}$ of $\cW$ satisfies $u_{\cW}(w) = \frac{\phi'(w)}{\|\phi'(w)\|_2}$ for $w\in\bd(\cW)$.

Next we compute the Weingarten map $W_w(v):\, T_w\cW \rightarrow T_w\cW$, which, by definition, is the differential of $u_{\cW}(w)$ restricted to $T_w\cW$. Note that the Weingarten map is a linear map.
\[
W_w(v) = \left.\frac{\mbox{d }u_{\cW} }{\mbox{d }w} \right\vert_{T_w\cW} (v) =  \frac{\nabla^2(w)v}{\|\phi'(w)\|_2} -\frac{\phi'(w)\nabla^2\phi(w)\phi'(w)^{T}v}{\|\phi'(w)\|_2^3} = \frac{\nabla^2(w)v}{\|\phi'(w)\|_2}.
\]

By \citep[page 105]{Sch14:ConvexBodies}, the principal curvature of $\cW$ at $w$ are the eigenvalues of the Weingarten map $W_w(v)$. Therefore, the smallest principal curvature at $w$ is $\min_{v\,:\,\|v\|_2=1, v\perp \phi'(w) }\frac{v^{\top}\nabla^2\phi(w) v}{\|\phi'(w)\|_2}$. Taking minimum over all  $w\in\bd(\cW)$ finishes the proof. 
\end{proof}

\begin{wrapfigure}{R}{0.35\textwidth}
	\vspace{-.05cm}
\begin{framed}
	\centering
	\includegraphics[width = \textwidth,
	trim={6.2cm 1cm 1.8cm 0},clip]
	{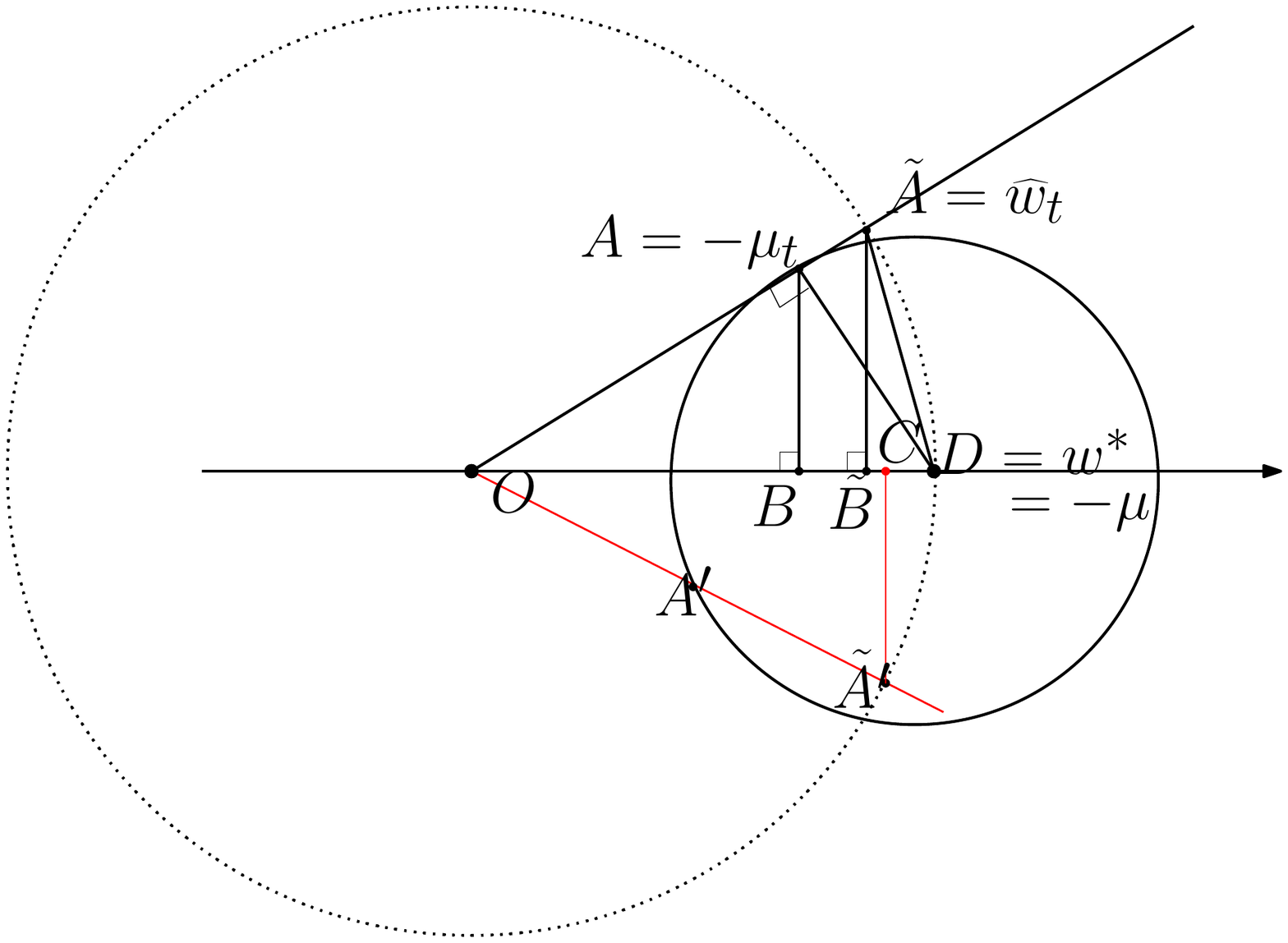}
	\vspace{-0.4cm}
	\caption{Illustration of how curvature helps to keep the regret small.
	}
	\label{fig:excesserror}
	\vspace{-0.1cm}
\end{framed}
	\vspace{-1.5cm}
\end{wrapfigure} 
In the stochastic i.i.d.\ case, when $\Exp{\Theta_t} = -\mu$, we have $\norm{\Theta_t +\mu}_2 = O(1/\sqrt{t})$ with high probability. 
Thus say, for $\cW$ being the unit ball of $\R^d$, one has $w_t = \Theta_t/\norm{\Theta_t}_2$; therefore,
a crude bound suggests that $\norm{w_t-  w^* }_2 = O(1/\sqrt{t})$, overall predicting that 
$\Exp{R_n} = O(\sqrt{n})$, while the previous result predicts that $R_n$ is much smaller.
In the next example we look at the unit ball, to explain geometrically, what ``causes'' the smaller regret.
\begin{example}
	\label{exam:ERM}
Let $\cW = \set{w}{\|w\|_2\le 1}$ and
consider a stochastic setting where the $f_i$ are i.i.d. samples 
from some underlying distribution with expectation $\Exp{f_i} = \mu = (-1,0,\ldots,0)$ and $\|f_i\|_\infty\le M$.
It is straightforward to see that $w^* = (1,0,\ldots,0)$, and thus $\inpro{w^*}{\mu} = -1$.
Let $E = \set{-\theta}{\|\theta - \mu\|_2 \le \epsilon}$. As suggested beforehand, we expect $-\mu_t\in E$ with high probability.
As shown in \cref{fig:excesserror}, 
the excess loss of an estimate $\vv{OA}$ is $\inpro{\vv{O\tilde{A}}}{\vv{OD}} - 1 = |\tilde{B}D|$.
Similarly, the excess loss of an estimate $\vv{OA'}$ in the figure is $|{CD}|$.
Therefore, for an estimate $-\mu_t \in E$, the point $A$ is where the largest excess loss is incurred.
The triangle $OAD$ is similar to the triangle $ADB$. Thus $\frac{|BD|}{|AD|} = \frac{|AD|}{|OD|}$. Therefore, 
$|BD| = \epsilon^2$ and since $|{\tilde{B}D}| \le |{BD}|$, 
if $\|\mu_t - \mu\|_2 \le \epsilon$, the excess error is at most $\epsilon^2 = O(1/t)$, making the regret $R_n = O(\log n)$.
\end{example}

Our last result in this section is an asymptotic lower bound for the linear game, showing that FTL achieves the optimal rate under the condition that  $\min_t \|\Theta_t\|_2\ge L >0$.
\begin{thm}
	\label{thm:lowerbound}
		Let $\lambda,L \in (0,1)$. Assume that  $\seto{(1,-L), (-1, -L)} \subset \cF$ and let \[
		\cW = \seto{(x,y) \in \R^2: x^2 + \frac{y^2}{\lambda^2} \le 1}
		\]
		be
		an ellipsoid with principal curvature $h$.
 		Then, for any learning strategy, there exists a sequence of losses in $\mathcal F$ such that $R_n = \Omega\left(\log(n)/(Lh)\right)$ and $\|\Theta_t\|_2 \ge L$ for all $t$.
\end{thm}

Note that by Example~\ref{ex:curvature}, the minimal principal curvature of $\cW$ in the above theorem is $\lambda$. In fact, it is not too hard to extend the above argument for any set $\cW$ such that there is $w \in \bd(\cW)$ where the curvature is $h$, and the curvature is a continuous function in a neighborhood of $w$ over the boundary $\bd(\cW)$. The constants in the bound then depend on how fast the curvature changes within this neighborhood.

\begin{proof}
 We define a random loss sequence, and we will show that no algorithm on this sequence can achieve an $o(\log n/ (\lambda_0 L)$ regret.
	Let $P$ be a random variable with $\mbox{Beta}(K,K)$ distribution for some $K>0$, and, given $P$, assume that $X_t, t \ge 1$ are i.i.d. Bernoulli random variables with parameter $P$. Let $f_t = X_t (1, -L) + (1-X_t) (-1, -L) = (2X_t - 1, -L)$. Thus, the second coordinate of $f_t$ is always $-L$, and so $\|\Theta_t\|_2 = \left\| \tfrac{1}{t} \sum_{i=1}^t f_i \right\|_2 \ge L$. Furthermore, the conditional expectation of the loss vector is $f^p \overset{\triangle}{=} \Expc{f_t}{P=p} = (2p - 1, -L)$. 
	
	Note that $X_t$ is a function of $f_t$ for all $t$; thus the conditional expectation of $P$, given $f_1,\ldots,f_{t-1}$, can be determined by the well-known formula $\hP_{t-1}= \Expc{P}{f_1 \ldots f_{t-1}} = \frac{K+\sum_{i=1}^{t-1} X_i}{2K+t-1}$.
	Given $p$, denote the optimizer of $f^p$ by $w^p$, that is, $w^p = \argmin_{w \in \cW} \inner{w,f^p}$. 	
	Then the Bayesian optimal choice in round $t$ is 
	\begin{align}
	\argmin_{w \in \mathcal W} \Expc{[\inner{w, f^P}}{ f_1\ldots f_{t-1}}
	&= \argmin_{w \in \mathcal W} \inner{w, \Expc{f^P}{f_1 \ldots f_{t-1}}} \nonumber \\
	&= \argmin_{w \in \mathcal W} \inner{w, f^{\hat P_{t-1}}} \nonumber \\
	&= w^{\hat P_{t-1}}\,,
	\label{eq:bayes-opt}
	\end{align}
	where the first equality follows by linearity of the inner product, the second since $f^p$ is a linear function of $p$ and the third
	by the definition of $w^p$.
	
	Thus, denoting by $W_t$ the prediction of an arbitrary algorithm in round $t$, the expected regret can be bounded from below as 
	\begin{align}
	\Exp{R_n}
	&= \Exp{\max_{w \in \cW} \sum_{t=1}^n \inner{W_t - w, f_t}}
	= \Exp{ \Expc{\max_{w \in \cW} \sum_{t=1}^n \inner{W_t - w, f_t}}{P} } \nonumber \\
	& \ge  \Exp{ \Expc{ \sum_{t=1}^n \inner{W_t - w^P, f_t}}{P} } = \Exp{\sum_{t=1}^n  \Expc{ \inner{W_t - w^P, f_t} }{P, f_1,\ldots,f_{t-1}}} \nonumber \\
	& = \Exp{\sum_{t=1}^n  \Expc{ \inner{W_t - w^P, f^P} }{f_1,\ldots,f_{t-1}}} \label{eq:Wf-ind} \\
	& \ge \Exp{\sum_{t=1}^n  \min_{w \in \cW} \Expc{ \inner{w- w^P, f^P} }{f_1,\ldots,f_{t-1}}} \nonumber \\
	& = \Exp{\sum_{t=1}^n  \Expc{ \inner{w^{\hP_{t-1}}- w^P, f^P} }{f_1,\ldots,f_{t-1}}}  \label{eq:bayes1} \\
	& = \sum_{t=1}^n \Exp{\inner{w^{\hP_{t-1}} - w^P, f^P}} \,, \nonumber 
	\end{align}
	where   \eqref{eq:Wf-ind} holds because of the independence of the $f_s$ given $P$ and since $W_t$ is chosen based on $f_1,\ldots,t_{t-1}$ (but not on $P$),
	and \eqref{eq:bayes1} holds by \eqref{eq:bayes-opt}. 
	
	By \cref{lem:P2P1loss} we have
	\begin{align}
	\sum_{t=1}^n \Exp{\inner{w^{\hP_{t-1}} - w^P, f^P}} 
	& \ge \frac{hL}{2}\sum_{t=1}^n \Exp{ \frac{\left( \frac{2\hP_{t-1} - 2P}{hL} \right)^2}{\sqrt{1+\left( \frac{1-2P}{hL}\right)^2 } \left(1+\left( \frac{1-2\hP_{t-1}}{hL}\right)^2 \right)} } \label{eq:hLloss} \\
	& = \frac{2}{hL}\sum_{t=1}^n \Exp{\frac{1}{\sqrt{1+\left( \frac{1-2P}{hL}\right)^2 }}\Expc{ \frac{ ( \hP_{t-1} - P)^2}{ 1+\left( \frac{1-2\hat P_{t-1}}{hL}\right)^2 } }{P} } \nonumber \\
	& \ge \frac{2}{hL}\sum_{t=1}^n \Exp{ \frac{1}{\sqrt{1+\left( \frac{1-2P}{hL}\right)^2 }}\Expc{ \frac{( \hP_{t-1} - P)^2}{ 1+ 2\left( \frac{1-2P}{hL}\right)^2 +2 \left(\frac{2P - 2\hP_{t-1}}{hL}\right)^2 }}{P } } \label{eq:hLlossCond}\,,
	\end{align}
	where in the last step we used $(a+b)^2 \le a^2 + b^2$. 	
	Let $\cG_t$ be the event that $|\hat P_{t} - P| \le \frac{K |1-2P|}{2K+t} + \frac{t hL}{2K+t}$; note that $\cG_t$ holds with high probability by \cref{lem:concenPhat}. Then, lower bounding the first term by $0$, \eqref{eq:hLlossCond} can be lower bounded by 
	\begin{align*}
	&\frac{2}{hL}\sum_{t=1}^{n-1} \Exp{ \frac{1}{\sqrt{1+\left( \frac{1-2P}{hL}\right)^2 }}\Expc{ \frac{( \hP_{t} - P)^2}{ 1+ 2\left( \frac{1-2P}{hL}\right)^2 +2 \left(\frac{2P - 2\hP_{t}}{hL}\right)^2 }\ind(\cG_t)}{P } } \\
	&\ge \frac{2}{hL}\sum_{t=1}^{n-1} \Exp{ \frac{1}{\sqrt{1+\left( \frac{1-2P}{hL}\right)^2 }}\frac{\Expc{ ( \hP_{t} - P )^2\ind(\cG_t) }{P}}{ \left(1+ 2\left( \frac{1-2P}{hL}\right)^2 +2 \left(\frac{2K}{2K+t}\frac{|1-2P|}{hL} + \frac{2t}{2K+t}\right)^2 \right)}  } \\
	& \ge \frac{2}{hL}\sum_{t=1}^{n-1} \Exp{\frac{1}{\sqrt{1+\left( \frac{1-2P}{hL}\right)^2 }}\frac{\Expc{ ( \hP_{t} - P )^2\ind(\cG_t) }{P}}{ \left(9+ 4\left( \frac{1-2P}{hL}\right)^2 +8 \frac{|1-2P|}{hL} \right)}  }.
	\end{align*}
	Combining the above, and using $(\hP_{t} - P )^2 \le 1$ together with the upper bound on the probability of the event $\cG^c_t$, the complement of $\cG_t$, given in \cref{lem:concenPhat}, we get
	\begin{align}
	\Exp{R_n} & \ge 
	\frac{2}{hL}\sum_{t=1}^{n-1} \Exp{\frac{1}{\sqrt{1+\left( \frac{1-2P}{hL}\right)^2 }}\frac{\Expc{ ( \hP_{t} - P )^2 }{P}-\Prob{\cG^c_t}}{ \left(9+ 4\left( \frac{1-2P}{hL}\right)^2 +8 \frac{|1-2P|}{hL} \right)}  } \nonumber \\
	& \ge \frac{2}{hL}\sum_{t=1}^{n-1} \left( \Exp{\frac{1}{\sqrt{1+\left( \frac{1-2P}{hL}\right)^2 }}\frac{\Expc{ ( \hP_{t} - P )^2 }{P}}{ \left(9+ 4\left( \frac{1-2P}{hL}\right)^2 +8 \frac{|1-2P|}{hL} \right)}  } - e^{-(t-1)h^2L^2} \right) \nonumber \\
	& \ge \frac{2}{hL}\left(\sum_{t=1}^{n-1} \Exp{\frac{1}{\sqrt{1+\left( \frac{1-2P}{hL}\right)^2 }}\frac{\Expc{ ( \hP_{t} - P )^2 }{P}}{ \left(9+ 4\left( \frac{1-2P}{hL}\right)^2 +8 \frac{|1-2P|}{hL} \right)}  }  \; -  \frac{1}{1-e^{-h^2L^2}} \right)\,. \label{eq:Rngc}
	\end{align} 
	Now, by \cref{lem:bayeserror}, we have
	\begin{align*}
	\Expc{ ( \hP_{t} - P )^2 }{P} & = \frac{K^2(1-2P)^2}{(2K+t)^2} + \frac{tP(1-P)}{(2K+t)^2} \ge P(1-P) \left( \frac{1}{t} - \frac{2}{t(2K+t)} \right)~. 
	\end{align*}
	Combining this with \eqref{eq:Rngc} and introducing the constant
	\[
	C = \Exp{\frac{1}{\sqrt{1+\left( \frac{1-2P}{hL}\right)^2 }}\frac{P(1-P)}{ \left(9+ 4\left( \frac{1-2P}{hL}\right)^2 +8 \frac{|1-2P|}{hL} \right)}  } 
	\]
	we obtain, for any $K>0$,
%	.Therefore, Picking $K=\frac{1}{h^2L^2}$,
	\begin{align}
	\liminf_{n \to \infty} \frac{\Exp{R_n}}{\log n} & \ge \liminf_{n \to \infty} \frac{2}{hL \log n}\left[ - \frac{1}{1-e^{-h^2L^2}} + \sum_{t=1}^{n-1} C\left(\frac{1}{t} - \frac{2}{t(2K+t)} \right) \right]
	 = \frac{2 C}{hL}~.
	\end{align}
	It remains to calculate a constant lower bound for $C$ that is independent of $h$ and $L$. Denote $\frac{|1-2P|}{hL}$ by $Y$; then $0\le P(1-P) = \frac{1-Y^2h^2L^2}{4}\le 1/4$. Define $\widehat{\cG}$ to be the event when 
	$|Y| \le 1$. Since $P$ has $\mbox{Beta}(K,K)$ distribution, $\Exp{P} = \frac{1}{2}$ and $\mbox{Var}(P) = \frac{1}{8K}$. Therefore, by Chebyshev's inequality,
	\begin{align*}
	\Prob{\widehat{\cG}^c} = \Prob{ \left| P-\frac{1}{2}\right| > \frac{hL}{2} } \le \frac{1}{2 K h^2 L^2}~.
	\end{align*}
Therefore,
	\begin{align*}
	C &= \Exp{\frac{1}{\sqrt{1+Y^2 }}\frac{1-Y^2h^2L^2}{ 4 (9+ 4Y^2 +8 Y )}}
	 \ge \Exp{\frac{1}{\sqrt{1+Y^2 }}\frac{1-Y^2h^2L^2}{ 4 (9+ 4Y^2 +8 Y )} \ind(\widehat{\cG}) } \\
	& \ge \frac{1}{84\sqrt{2}}\Exp{(1-Y^2h^2L^2)\ind(\widehat{\cG})} 
	 \ge \frac{1}{84\sqrt{2}} \left( \Exp{1-Y^2h^2L^2} - \Prob{\widehat{\cG}^c}\right) \\
	& \quad \ge \frac{1}{84\sqrt{2}} \left( 1- \Exp{(1-2P)^2} - \frac{1}{2Kh^2L^2}\right)
	=  \frac{1}{84\sqrt{2}} \left(\frac{1}{2} - \frac{h^2L^2}{2}\right).
	\end{align*}
	Therefore, 
	\[
	\liminf\limits_{n\rightarrow \infty} \frac{\Exp{R_n}}{\log n} \ge \frac{1}{84\sqrt{2}}\left(\frac{1}{hL} - hL\right) \ge \frac{1}{84\sqrt{2}}\left(\frac{1}{hL} - 1\right).
	\]
	The result is completed by noting that the worst-case regret is at least as big as the expected regret, thus, for every $n$, there exist a $P$ and a sequence of loss vectors $f_1,\ldots,f_n$ such that the regret $R_n$ is at least $\Omega(\frac{\log n }{hL})$.
\end{proof}

\subsection{Other regularities}

So far we have looked at the case when FTL achieves a low regret due to the curvature of $\bd(\cW)$.
The next result characterizes the regret of FTL when $\cW$ is a polytope, which has a flat, non-smooth boundary and thus \cref{thm:R_curvesurface} is not applicable. 
For this statement recall that given some norm $\|\cdot\|$,  its dual norm is defined by $\|w\|_* = \sup_{\|v\|\le 1} \inpro{v}{w}$.
\begin{thm}
	\label{thm:regretpolytope}
	Assume that $\cW$ is a polytope
	and that $\Phi$ is differentiable at $\Theta_i$, $i= 1, \ldots, n$. 
	Let $w_t = \argmax_{w\in\cW} \inpro{w}{\Theta_{t-1}}$,
	$W = \sup_{w_1,w_2\in\cW}\|w_1 - w_2\|_*$ and $F = \sup_{f_1,f_2\in \cF} \norm{f_1-f_2}$.
	 Then the regret of FTL is 
	\[
	R_n \le W\, \sum_{t=1}^{n} t \,\ind(w_{t+1}\neq w_{t})  \|\Theta_t - \Theta_{t-1}\| \le FW\,\sum_{t=1}^{n} \ind(w_{t+1}\neq w_{t})\,.
	\]
\end{thm}
Note that when $\cW$ is a polytope, $w_t$ is expected to ``snap'' to some vertex of $\cW$. Hence, 
we expect the regret bound to be non-vacuous, if, e.g., $\Theta_t$ ``stabilizes'' around some value. Some examples after the 
proof will illustrate this.
\begin{proof}
Let $v \!=\! \argmax_{w\in\cW} \inpro{w}{\theta}$, $v'\!=\!\argmax_{w\in \cW}\ip{w,\theta'}$. 
Similarly to the proof of \cref{thm:R_curvesurface},
	\begin{align*}
	\inpro{v'-v}{\theta'} & = \inpro{v'}{\theta'} - \inpro{v'}{\theta} + \inpro{v'}{\theta} - \inpro{v}{\theta} + \inpro{v}{\theta} -\inpro{v}{\theta'} \\
	& \le \inpro{v'}{\theta'} - \inpro{v'}{\theta} + \inpro{v}{\theta} -\inpro{v}{\theta'} %\numberthis \label{eq:eq211}
	= \inpro{v' - v}{\theta' - \theta} 
	%\\& 
	\le W\,\ind(v'\neq v)\|\theta' - \theta \|,
	\end{align*}
	where the first inequality %~\eqref{eq:eq211} 
	holds because $\inpro{v'}{\theta} \le \inpro{v}{\theta}$.
	Therefore, by \cref{prop:avgdiff}, 
	\begin{align*}
	R_n & = \sum_{t=1}^n t\,\ip{ w_{t+1}-w_t,\Theta_t} 
	 \le W\,\sum_{t=1}^{n} t\, \ind(w_{t+1}\!\neq\! w_{t})  \|\Theta_t - \Theta_{t-1}\| 
	 \le FW\,\sum_{t=1}^{n} \ind(w_{t+1}\!\neq\! w_{t})\,.
	\end{align*}
\end{proof}
\if0
\begin{comm}
	\cref{thm:regretpolytope} bounds the regret of FTL by the number of switches of the maximizers $\sum_{t=1}^{n} \ind(w_t\neq w_{t-1})$.
\end{comm}
\fi
%\begin{comm}
As noted before,  since $\cW$ is a polytope, $w_t$ is (generally) attained at the vertices. 
In this case, the epigraph of $\Phi$  is a polyhedral cone. Then, the event when $w_{t+1}\neq w_{t}$, i.e., when 
	the ``leader'' switches corresponds to when 
	$\Theta_{t}$ and $\Theta_{t-1}$ belong to different linear regions corresponding to different linear pieces of the graph of $\Phi$.
%\end{comm}

We now spell out a corollary for the stochastic setting. In particular, in this case FTL will often enjoy a constant regret:
\begin{cor}[Stochastic setting]
	\label{cor:stocpolytope} Assume that $\cW$ is a polytope and 
	that $(f_t)_{1\le t \le n}$ is an i.i.d. sequence of random variables 
	such that $\Exp{f_i} = \mu$ and $\|f_i\|_\infty \le M$. Let  $W = \sup_{w_1,w_2\in \cW} \norm{w_1-w_2}_1$.
	Further assume that there exists a constant $r > 0$ 
	such that $\Phi$ is differentiable for any $\nu$ such that $\|\nu-\mu\|_\infty \le r$. 
	\todoc{We should probably explain the intuitive meaning of this. Maybe replace this with something 
	more intuitive in fact.. $r$ should be the radius of the largest ball such that $\nu$ and $\mu$ are on the same
	face of $\cW$. Then we won't need $\Phi$ indeed.} \todoa{Explained after the corollary.}
	Then, % there exists a constant $C$, such that 
	\[
		\Exp{R_n} \le 2MW \, (1+4d M^2/r^2 )\,.
%	O(\sum_{t=1}^{n} \Prob{-\Theta_t \notin V}) = O(1).
	\]
\end{cor}
The condition on $\Phi$ means that $r$ can be selected to be the radius of the largest ball such that the optimal decisions for expected losses $\mu$ and $\nu$ (i.e., the maximizers defining $\Phi(-\mu)$ and $\Phi(-\nu)$) belong to the same face of $\cW$.
\begin{proof}
	Let $V = \set{\nu}{\|\nu - \mu\|_\infty\le r}$. 
	Note that the epigraph of the function $\Phi$ is a polyhedral cone.
	Since $\Phi$ is differentiable in the interior of $V$, $\set{(\theta, \Phi(\theta))}{\theta\in V}$ is a subset of a linear subspace.
	Therefore, for $-\Theta_t, -\Theta_{t-1} \in V$, $w_{t+1}=w_t$.
	Hence, by \cref{thm:regretpolytope},
	\[
	\Exp{R_n} \le 2MW\,\sum_{t=1}^{n} \Prob{-\Theta_t,-\Theta_{t-1} \notin V}
	 \le 4MW\,\left(1+\sum_{t=1}^{n} \Prob{-\Theta_t \notin V}\right)\,.
	\]
	On the other hand, note that $\|f_i\|_\infty\le M$.
	Then 
	\begin{align*}
	\Prob{-\Theta_t \notin V}
	    & = \Prob{ \norm{\frac{1}{t} \sum_{i=1}^{t} f_i - \mu}_\infty \ge r}
		 \le \sum_{j=1}^{d} \Prob{ \left|\frac{1}{t} \sum_{i=1}^{t} f_{i,j} - \mu_j\right| \ge r }
		 \le 2d e^{-\frac{tr^2}{2M^2}}\,,
	\end{align*}
	where the last inequality
	%~\eqref{eq:eq212} 
	is due to Hoeffding's inequality.
	Now, using that for $\alpha>0$, $\sum_{t=1}^n \exp(-\alpha t ) \le \int_0^n \exp(-\alpha t ) dt 
	\le \frac{1}{\alpha}$, we get
$
	\Exp{R_n} \le 2MW \, (1+4d M^2/r^2 )
$.
\end{proof}

	The condition that $\Phi$ is differentiable for any $\nu$ such that $\|\nu-\mu\|_\infty \le r$ is equivalent to that $\Phi$ is differentiable at $\mu$. 
	By \cref{prop:derivativePhi}, this condition requires that at $\mu$, $\max_{w\in\cW} \ip{w,\theta}$ has a unique optimizer.
	Note that the volume of the set of vectors $\theta$ with multiple optimizers is zero. 

\section{Adaptive algorithm for the linear game}
While as shown in \cref{thm:R_curvesurface}, FTL can exploit the curvature of the surface of the constraint set to achieve $O(\log n)$ regret, it requires the curvature condition and $\min_t \|\Theta_t\|_2 \ge L$ being bounded away from zero, or
 it may suffer even linear regret.
On the other hand, many algorithms,  such as the "Follow the regularized leader" (FTRL) algorithm \citep[see,e.g.,][]{SS12:Book}, are known to achieve a regret guarantee of $O(\sqrt{n})$ even for the worst-case data in the linear setting.
This raises the question whether one can have an algorithm that can 
achieve constant or $O(\log n)$ regret in the respective settings of  \cref{cor:stocpolytope} or \cref{thm:R_curvesurface},
while it still maintains $O(\sqrt{n})$ regret for worst-case data. 
One way to design an adaptive algorithm is to use the ($\cA$, $\cB$)-prod algorithm of \citet{sani2014exploiting}, trivially leading to the following result:
\begin{prop}
Consider ($\cA$, $\cB$)-prod of \citet{sani2014exploiting}, where algorithm \todoa{Do we need to write out ($\cA$, $\cB$)-prod?}
 $\cA$ is chosen to be FTRL with an appropriate regularization term, 
 while $\cB$ is chosen to be FTL. 
Then the regret of the resulting hybrid algorithm $\cH$ enjoys the following guarantees:
%(i) If FTL achieves constant regret as in the setting of \cref{cor:stocpolytope}, then the regret of $\cH$ is also constant.
%(ii) If FTL achieves a regret of $O(\log n)$ as in the setting of \cref{thm:R_curvesurface}, then the regret of $\cH$ is also $O(\log n)$.
%(iii) Otherwise, the regret of $\cH$ is at most $O(\sqrt{n\log n})$.
\begin{itemize}\setlength{\itemsep}{0pt}
\item If FTL achieves constant regret as in the setting of \cref{cor:stocpolytope}, then the regret of $\cH$ is also constant.
\item If FTL achieves a regret of $O(\log n)$ as in the setting of \cref{thm:R_curvesurface}, then the regret of $\cH$ is also $O(\log n)$.
\item Otherwise, the regret of $\cH$ is at most $O(\sqrt{n\log n})$.
\end{itemize}
\end{prop} 

In the next section we show that if the constraint set is the unit ball, it is possible to design adaptive algorithms directly.

\subsection{Adaptive Algorithms for the Unit Ball Constraint Set}

In this section we provide some interesting results about adaptive algorithms for the case when $\cW$ is the unit ball in $\R^d$ (naturally, the results easily generalize to any ball centered at the origin). First, we show that a variant of FTL using shrinkage as regularization has $O(\log(n))$ regret when $\|\Theta_t\|_2 \ge L>0$ for all $t$, but it also has $O(\sqrt{n})$ worst case guarantee. Furthermore, we show that the standard FTRL algorithm is adaptive if the constraint set is the unit ball and the loss vectors are stochastic.
Throughout the section we will use the notation $F_t=-(t-1)\Theta_t=\sum_{i=1}^{t-1} f_i$.

\subsubsection{Follow the Shrunken Leader}

In this section we are going to analyze a combination of the FTL algorithm and the idea of shrinkage often used for regularization purposes in statistics. We assume that $\cW=\set{x \in \R^d}{ \|x\|_2 \le 1}$ is the unit ball and,  without loss of generality, we further assume that $\|f\|_2 \le 1$ for all $f\in\cF$. 
	\begin{algorithm}[t]
		\caption{Follow The Shrunken Leader (FTSL}
		\label{alg:adaptiveAlgorithm}
		\begin{algorithmic}[1]
			\STATE Predict $w_1 = 0$; 
			\FOR {$t = 2, ..., n-1$}
			\STATE {FTL: Compute $\tilde{w}_{t} = \argmin_{w\in\cW} \inner{ w, F_{t-1}}$}
			\STATE {Shrinkage: Predict $w_t = \frac{\|F_{t-1}\|_2}{\sqrt{\|F_{t-1}\|_2^2+t+2}}\tilde{w}_{t}$}
			\ENDFOR
			\STATE {FTL: Compute $\tilde{w}_{n} = \argmin_{w\in\cW} \inner{ w, F_{n-1}}$}
			\STATE {Shrinkage: Predict $w_n = \frac{\|F_{n-1}\|_2}{\sqrt{\|F_{n-1}\|_2^2+n}}\tilde{w}_{n}$}
		\end{algorithmic}
	\end{algorithm}
\begin{thm}
 
 The Follow The Shrunken Leader (FTSL) algorithm is given in \cref{alg:adaptiveAlgorithm}. The main idea of the algorithm is to predict a shrunken version of the FTL prediction, in this way keeping it away from the boundary of $\cW$. The next theorem shows that the right amount of shrinkage leads to a robust, adaptive algorithm:
  \begin{itemize}
 	\item If there exists $L$ such that $\|\Theta_t\|_2 \ge L>0$ for $1\le t\le n$, then the regret of FTSL is $O(\log(n)/L)$.
 	\item Otherwise, the regret of FTSL is at most $O(\sqrt{n})$.
 \end{itemize}
\end{thm}
\begin{proof}
	By the definition of $F_t$ and $\cW$, $\tilde{w}_{t} =- F_{t-1}/\|F_{t-1}\|_2$.
	Let $\sigma_n = \frac{\|F_{n-1}\|_2}{\sqrt{\|F_{n-1}\|_2^2 + n}}$.
	Our proof follows the idea of \citet{abernethy2008optimal}. We compute the upper bound on the value of the game for each round backwards for $t=n,n-1,\dots,1$, by solving the optimal strategies for $f_t$.
	The value of the game using FTSL is defined as
	\begin{align*}
	V_n & = \max_{f_1, \ldots, f_n} \sum_{t=1}^{n}\inpro{w_t}{f_t}- \min_{w\in\cW} \inpro{w}{F_n} \\
	& = \max_{f_1,\ldots, f_{n-1}} \sum_{t=1}^{n-1}\inpro{w_t}{f_t} + \underbrace{\max_{f_n} \|F_{n-1}+f_n\|_2 + \inpro{f_n}{w_n}}_{=:U_n}
	\end{align*}
We first prove that $U_n$, the second term above, is bounded from above by  $\sqrt{\|F_{n-1}\|_2^2 + n}$. To see this, let $f_n = a_n \tilde{F}_{n-1} + b_n \Omega_{n-1}$ where $\tilde{F}_{n-1}$ is the unit vector parallel to $F_{n-1}$ and $\Omega_{n-1}$ is a unit vector orthogonal to $F_{n-1}$.  Furthermore, since $\|f_n\|_2 \le 1$, we have $a_n^2+b_n^2 \le 1$.
Thus,
\begin{align*}
U_n & = \max_{f_n} \sqrt{\|F_{n-1}\|_2^2 + 2a_n\|F_{n-1}\|_2 + a_n^2 + b_n^2} - a_n\sigma_{n}\\
 & \le \max_{a} \sqrt{\|F_{n-1}\|_2^2 + 2a\|F_{n-1}\|_2 + n} - a\sigma_{n}\\
 & = \sqrt{\|F_{n-1}\|_2^2 + n},
\end{align*}  
where the last equality follows since the maximum is attained at $a=0$. 
A similar statement holds for the other time indices: for any $t \ge 1$,
\begin{equation}
\label{eq:stepDiff1}
\max_{f_t} \sqrt{\|F_{t-1} + f_t\|_2^2 + t + 1} + \inpro{f_t}{w_t} \le  \sqrt{\|F_{t-1}\|_2^2 + t} + \frac{1}{\sqrt{t}}~.
\end{equation}
Before proving this inequality, let us see how it implies the second statement of the theorem:
\begin{align*}
V_n & \le \max_{f_1,\ldots, f_{n-1}} \sum_{t=1}^{n-1}\inpro{w_t}{f_t} + \sqrt{\|F_{n-1}\|_2^2 + n} \\
& \le \max_{f_1,\ldots, f_{n-2}} \sum_{t=1}^{n-2}\inpro{w_t}{f_t}  + \sqrt{\|F_{n-2}\|_2^2 + n-1} + \frac{1}{\sqrt{n}} \\
& \le \ldots \\
& \le 1+ \sum_{t=1}^{n}\frac{1}{\sqrt{t}} = O(\sqrt{n}).
\end{align*}
Moreover, if $\|\Theta_t\|_2 \ge L$ for $1\le t\le n$, a stronger version of \eqref{eq:stepDiff1} also holds:
\begin{equation}
\label{eq:stepDiff2}
\max_{f_t} \sqrt{\|F_{t-1} + f_t\|_2^2 + t + 1} + \inpro{f_t}{w_t} \le  \sqrt{\|F_{t-1}\|_2^2 + t} + \frac{1}{(t-1)L}.
\end{equation}
This implies the first statement of the theorem, since
\begin{align*}
V_n & \le \max_{f_1,\ldots, f_{n-1}} \sum_{t=1}^{n-1}\inpro{w_t}{f_t} + \sqrt{\|F_{n-1}\|_2^2 + n} \\
& \le \max_{f_1,\ldots, f_{n-2}} \sum_{t=1}^{n-2}\inpro{w_t}{f_t}  + \sqrt{\|F_{n-2}\|_2^2 + n-1} + \frac{1}{(n-1)L} \\
& \le \ldots \\
& \le 1+ \sum_{t=1}^{n-1}\frac{1}{tL} = O(\log(n)/L).
\end{align*}

To finish the proof, it remains to show \eqref{eq:stepDiff1} and \eqref{eq:stepDiff2}.
Let  $f_t = a_t \tilde{F}_{t-1} + b_t \Omega_{t-1}$ where $\tilde{F}_{t-1}$ is the unit vector parallel to $F_{t-1}$ and $\Omega_{t-1}$ is a unit vector orthogonal to $F_{t-1}$. Since $\|f_t\|_2 \le 1$, observe that $a_t^2+b_t^2 =\|f_t\|_2 \le 1$. Furthermore, let $\sigma_t = \frac{\|F_{t-1}\|_2}{\sqrt{\|F_{t-1}\|_2^2 + t+2}}$.
Then, for any $t \ge 1$,
	\begin{align}
	\Delta_t & =\max_{f_t} \sqrt{\|F_{t-1}\|_2^2 + 2a_t\|F_{t-1}\|_2 + a_t^2 + b_t^2 + t+1} - a_t\sigma_t  - \sqrt{\|F_{t-1}\|_2^2 + t}  \nonumber\\
	& \le \max_{a_t} \sqrt{\|F_{t-1}\|_2^2 + 2a_t\|F_{t-1}\|_2 + t+2} - a_t\sigma_t  - \sqrt{\|F_{t-1}\|_2^2 + t}  \nonumber\\
	& = \sqrt{\|F_{t-1}\|_2^2 + t+2} - \sqrt{\|F_{t-1}\|_2^2 + t}  \nonumber\\
	& = \frac{2}{\sqrt{\|F_{t-1}\|_2^2 + t+2} + \sqrt{\|F_{t-1}\|_2^2 + t}} \label{eq:stepDiff3} \\
	& \le \frac{1}{\sqrt{t}}. \nonumber
	\end{align}
	This proves \eqref{eq:stepDiff1}.
	Moreover, if $\|F_{t-1}\|_2 = \|(t-1)\Theta\|_2 \ge (t-1)L$, by \eqref{eq:stepDiff3} we obtain
	\[
	\Delta_t \le \frac{2}{\sqrt{\|F_{t-1}\|_2^2 + t+2} + \sqrt{\|F_{t-1}\|_2^2 + t}} \le \frac{1}{\|F_{t-1}\|_2}\le \frac{1}{(t-1)L},
	\]
	proving \eqref{eq:stepDiff2}.
\end{proof}

\subsubsection{FTRL for the case of the unit ball constraint set}
This section is to show that in the case when $\cW$ is the unit ball in $\ell_2$ norm, FTRL  with $R(w) = \frac{1}{2}\|w\|^2$ as its regularization is an adaptive algorithm. To fix the notation, in round $t$, FTLR predicts
\[
	w_{t} = \argmin_{w\in \cW} \eta_t \inpro{F_{t-1}}{w} + R(w), 
\]
if $t >1$ and $w_1=0$.
It has been well known that FTRL with $\eta_t = 1/\sqrt{t-1}$ is guaranteed to achieve $O(\sqrt{n})$ regret in the adversarial setting, see, e.g., \citep{SS12:Book}. It remains to prove that FTRL indeed achieves a fast rate in the stochastic setting. 
\begin{thm}
	Assume that the sequence of loss vectors, $f_1,\ldots,f_n \in \R^d$  satisfies $\|f_t\|_2 \le 1$ almost surely and $\Exp{f_t} = \mu$ for all $t$ with some $\|\mu\|_2 >0$. Then FTRL with $\eta_t=1/\sqrt{t-1}$ suffers $O(\log n)$ regret .
\end{thm}

\begin{proof}
	Using $R(w) = \frac{1}{2}\|w\|^2$ as its regularization, in round $t>1$ FTRL predicts 
	\begin{equation}
	\label{eq:ftrl-eq}
	w_{t} = \argmin_{w\in \cW} \eta_t \inpro{F_{t-1}}{w} + R(w) = 
	\begin{cases}
	\frac{1}{\sqrt{t-1}} F_{t-1} & \quad \text{if  } \|F_{t-1}\| \le \sqrt{t-1} \\
	\frac{F_{t-1}}{\|F_{t-1}\|} & \quad \text{otherwise.} 
	\end{cases}
	\end{equation}
	For any $1\le t \le n$,  denote the event $\|F_t\| \ge \sqrt{t}$ by $\cE_t$. Note that if $\|F_{t-1}\| \ge \sqrt{t-1}$, FTRL predicts exactly the same $w_t$ as FTL. Denote the accumulate loss of FTL in $n$ rounds by $\cL^{FTL}_n$. Thus, the regret of FTRL is
	\begin{align*}
	\Exp{R_n} & = \Exp{\sum_{t=1}^{n} \inpro{f_t}{w_t} - \min_{w\in\cW}  \inpro{f_t}{w}} \\
	& = \Exp{ \sum_{t=1}^{n} \inpro{f_t}{w_t} - \cL^{FTL}_n }+ \Exp{\cL^{FTL}_n - \min_{w\in\cW} \inpro{f_t}{w} } \\
	& \le 2 \sum_{t=1}^{n} \Prob{\cE_t^c} + O(\log n),
	\end{align*}
	where, to obtain the last inequality, we applied \eqref{eq:ftrl-eq} for the first term, while the second term is $O(\log n)$ by the discussion following \cref{thm:R_curvesurface}.
	It remains to bound the first term, 2 $\sum_{t=1}^{n} \Prob{\cE_t^c} $ in the above.
	For any $t > \frac{4}{\|\mu\|_2^2}$,
	\begin{align*}
	\Prob{ \|F_{t}\|_2 \le \sqrt{t} } &\le \Prob{ \|F_{t}\|_2 <  \frac{t}{2}\|\mu\|_2 } \le \sum_{i=1}^d \Prob{|F_{t,i}| < \frac{t}{2} |\mu_i|} \\
	             &\le \sum_{i=1}^d \Prob{|F_{t,i}-t\mu_i| >  \frac{t}{2} |\mu_i|} \le 2 \sum_{i=1}^d e^{-\frac{\mu_i^2}{4} t}
	\end{align*}
	Thus,
	\begin{align*}
	\sum_{t=1}^{n} \Prob{\cE_t^c}  & = \sum_{t=1}^{4/\|\mu\|_2^2} \Prob{\cE_t^c}  + \sum_{t=4/\|\mu\|_2^2}^{n} \Prob{\cE_t^c} \\
	& \le \frac{4}{\|\mu\|_2^2} + 2\sum_{i=1}^d \sum_{t=0}^{n} e^{-\frac{\mu_i^2}{4} t} \\
	& \le \frac{4}{\|\mu\|_2^2} + 2\sum_{i=1}^d \frac{1}{1-e^{ -\frac{\mu_i^2}{4}}} \\
	& \le  \frac{4}{\|\mu\|_2^2} + 2\sum_{i=1}^d \frac{\mu_i^2}{4} = \frac{4}{\|\mu\|_2^2}  + \frac{\|\mu\|_2^2}{2}~.
	\end{align*}
	where in the last inequality we used $1/(1-e^{-a}) \le a$.
	Therefore, if $\|\mu\| > 0$, the regret of FTRL satisfies
	\[
	\Exp{R_n} \le \frac{8}{\|\mu\|_2^2} + \|\mu\|_2^2 + O(\log n) = O(\log n).
	\]
\end{proof}

\section{Simulations}
\label{sec:Simulations}
We performed three simulations to illustrate the differences between  FTL, FTRL with the regularizer $R(w) = \frac12 \norm{w}_2^2$ when
$w_t = \argmin_{w\in \cW} \sum_{i=1}^{t-1} \ip{f_{i-1},w} + R(w)$,
and the adaptive algorithm ($\cA$, $\cB$)-prod (AB) using FTL and FTRL as its candidates, which we shall call AB(FTL,FTRL).

For the experiments the constraint set $\cW$ was chosen to be a slightly elongated ellipsoid in the $4$-dimensional Euclidean space, with volume matching that of the $4$-dimensional unit ball.
The actual ellipsoid is given by 
$\cW = \set{w\in \R^4}{w^{\top}Qw \le 1}$
where $Q$ is randomly generated as
\[
Q  = \left(\begin{array}{cccc}
4.3367    & 3.6346   & -2.2250   & 3.5628 \\
3.6346    & 3.9966   & -2.3613   & 3.2817\\
-2.2250   & -2.3613  &  2.0589  & -2.1295\\
3.5628    & 3.2817  & -2.1295  &  3.4206\\
\end{array}\right).
\]

We experimented with 3 types of data to illustrate the behavior of the different algorithms: stochastic, ``half-adversarial'', and ``worst-case'' data (worst-case for FTL), as will be explained below.
The first two datasets are random, so the experiments were repeated 100 times, and we report the average regret with its standard deviation; the worst case data is deterministic, so there no repetition was needed.
For each experiment, we set $n = 2500$. 
The regularization coefficient for the FTRL, and the learning rate for AB were chosen based on their theoretical bounds
minimizing the worst-case regret.
%so as to ensure that the regret is of order $O(\sqrt{n})$.
%We  used the naive ($\cA$, $\cB$)-prod algorithm of \citet{sani2014exploiting}. 

\paragraph{Stochastic data.}
In this setting  we used the following model to generate $f_t$:
Let  $(\hat{f}_t)_t$ be an i.i.d. sequence drawn from the 4-dimensional standard normal distribution, and let $\tilde{f}_t = \hat{f}_t/\norm{\hat{f}_t}_2$.
Then, $f_t$ is defined as $f_t = \tilde{f}_t  + L e_1$ where $e_1 = (1,0,\dots,0)^\top$. 
Therefore, $\Exp{\norm{\tfrac{1}{t}\sum_{s=1}^t f_s}_2} \to L$ as $t \to \infty$.
In the experiments we picked $L \in \{0, 0.1\}$.

The results are shown in \cref{res:stoch}.
On the left-hand side we plotted the regret against the logarithm of the number of rounds, while on the right-hand side
we plotted the regret against the square root of the number of rounds, together with the standard deviation of the results
over the $100$ independent runs.
As can be seen from the figures,
when $L=0.1$, the growth-rate of the regret of FTL is indeed logarithmic, while when $L=0$, the growth-rate is
$\Theta(\sqrt{n})$. In particular, when $L=0.1$, FTL enjoys a major advantage compared to FTRL,
while for $L=0$, FTL and FTRL perform essentially the same (in this special case, the regret of FTL will indeed 
be $O(\sqrt{n})$ as $w_t$ will stay bounded but $\norm{\Theta_t} = O(1/\sqrt{t})$).
As expected, AB(FTL,FTRL), gets the better of the two regrets with little to no extra penalty.
\todoa{Wouldn't it be enough to provide one picture for both values of $L$ using the relevant scale? I don;t know which one is better.}

\begin{figure}[th]
	\centering
	\includegraphics[width=0.8\textwidth]{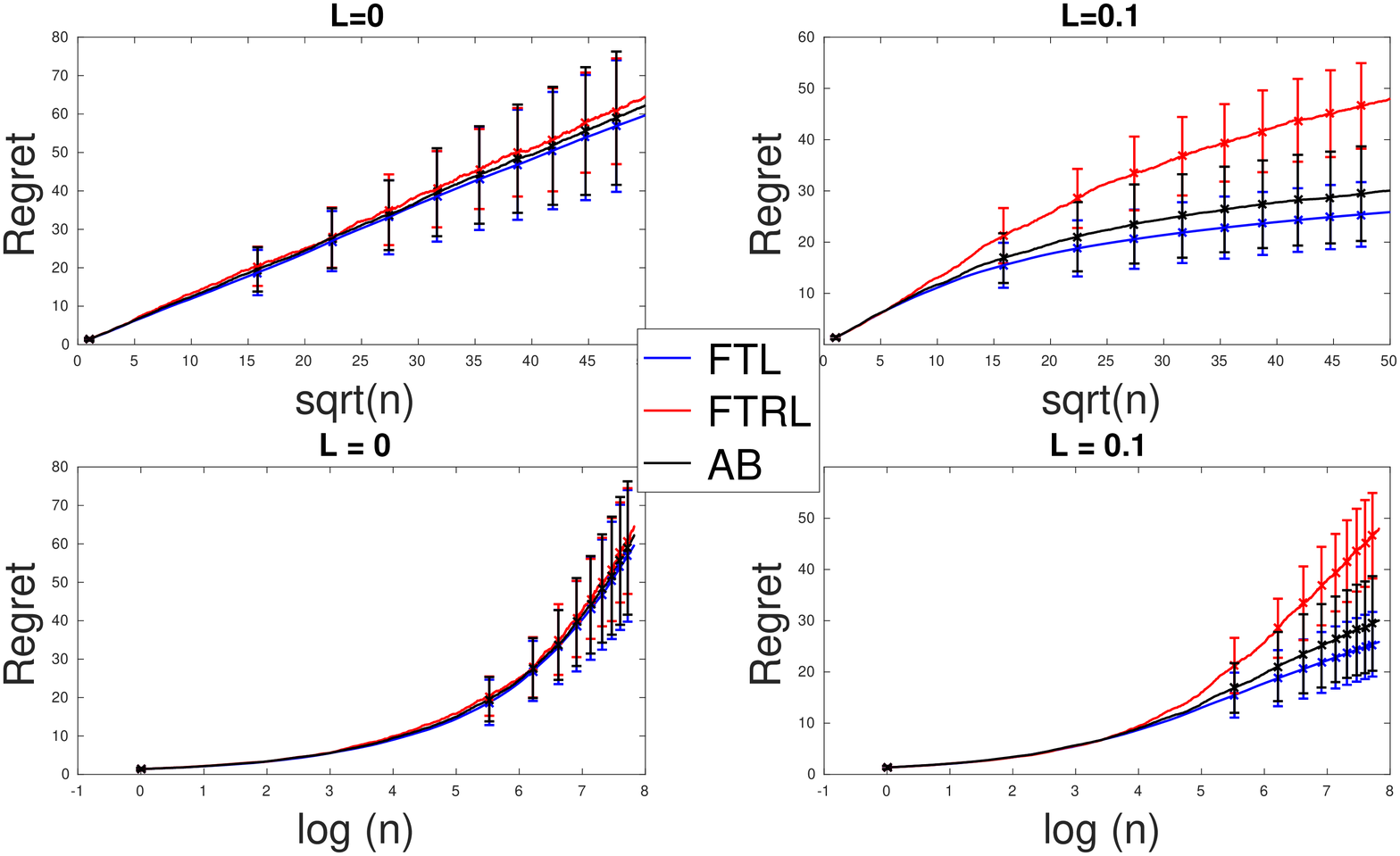}
	\caption{Regret of FTL, FTRL and AB(FTL,FTRL) against time for stochastic data. \label{res:stoch}}
\end{figure}

\paragraph{``Half-adversarial'' data}
The half-adversarial data used in this experiment is the optimal solution for the adversary 
in the \emph{linear game} when $\cW$ is the unit ball \citep{abernethy2008optimal}. 
This data is generated as follows:
The sequence $\hat{f}_t$ for $t = 1, \ldots, n$ is generated randomly
in the $(d-1)$-dimensional subspace $S = \text{span}\{e_2, \ldots, e_d\}$ (here $e_i$ is the $i$th unit vector in $\R^d$) as follows:
$\hat{f}_1$ is drawn from the uniform distribution on the unit sphere of $S$ (actually $\bS_{d-2}$. 
For $t = 2, \ldots, n$, $\hat{f}_t$ is drawn from the uniform distribution on the unit sphere
of the intersection of $S$ and the hyperplane perpendicular to $\sum_{i=1}^{t-1} \hat{f}_i$ and going through the origin.
Then, $f_t = Le_1 + \sqrt{1-L^2} \hat{f}_t$ for some $L \ge 0$.
%Note that for $L>0$, the resulting sequence is the one used in proving our logarithmic lower bound.   A: Not anymore

The results are reported in \cref{res:adver}.
When $L=0$,  the regret of both FTL and FTRL grows as $O(\sqrt{n})$. 
When $L=0.1$, FTL achieves $O(\log n)$ regret,
while the regret of FTRL appears to be $O(\sqrt{n})$. 
AB(FTL,FTRL) closely matches the regret of FTL.

\begin{figure}[th]
	\centering
	\includegraphics[width=0.8\textwidth]{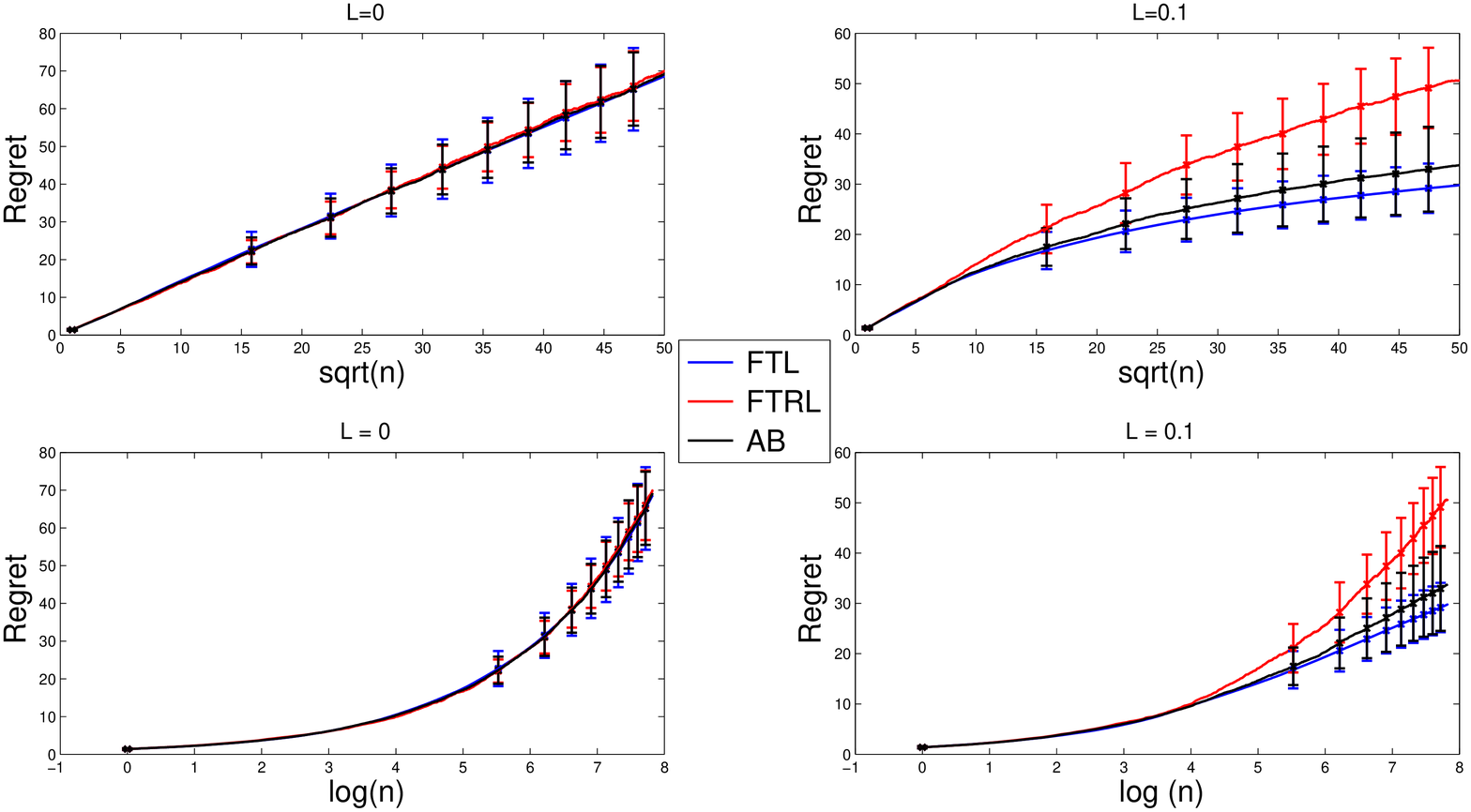}
	\caption{Experimental results for ``half-adversarial'' data. \label{res:adver}}
\end{figure}

\paragraph{Worst-case data}
We also tested the algorithms on data where FTL is known to suffer linear regret, mainly to see how well AB(FTL,FTRL) is able to deal with this setting.
In this case, we set $f_{t,i}=0$ for all $t$ and $i\ge 2$, while 
for the first coordinate, $f_{1,1} = 0.9$, and $f_{t,1} = 2(t \mod 2) - 1$ for $t \ge 2$.

The results are reported in \cref{res:worst_case}. It can be seen that the regret of FTL is linear (as one can easily verify theoretically), and 
AB(FTL,FTRL) succeeds to adapt to FTRL, and they both achieve a much smaller $O(\sqrt{n})$ regret. \todoa{The scale of the axes should be sqrt-sqrt and sqrt-lin to show the required dependencies.}

\begin{figure}[th]
	\centering
	\includegraphics[width=0.8\textwidth]{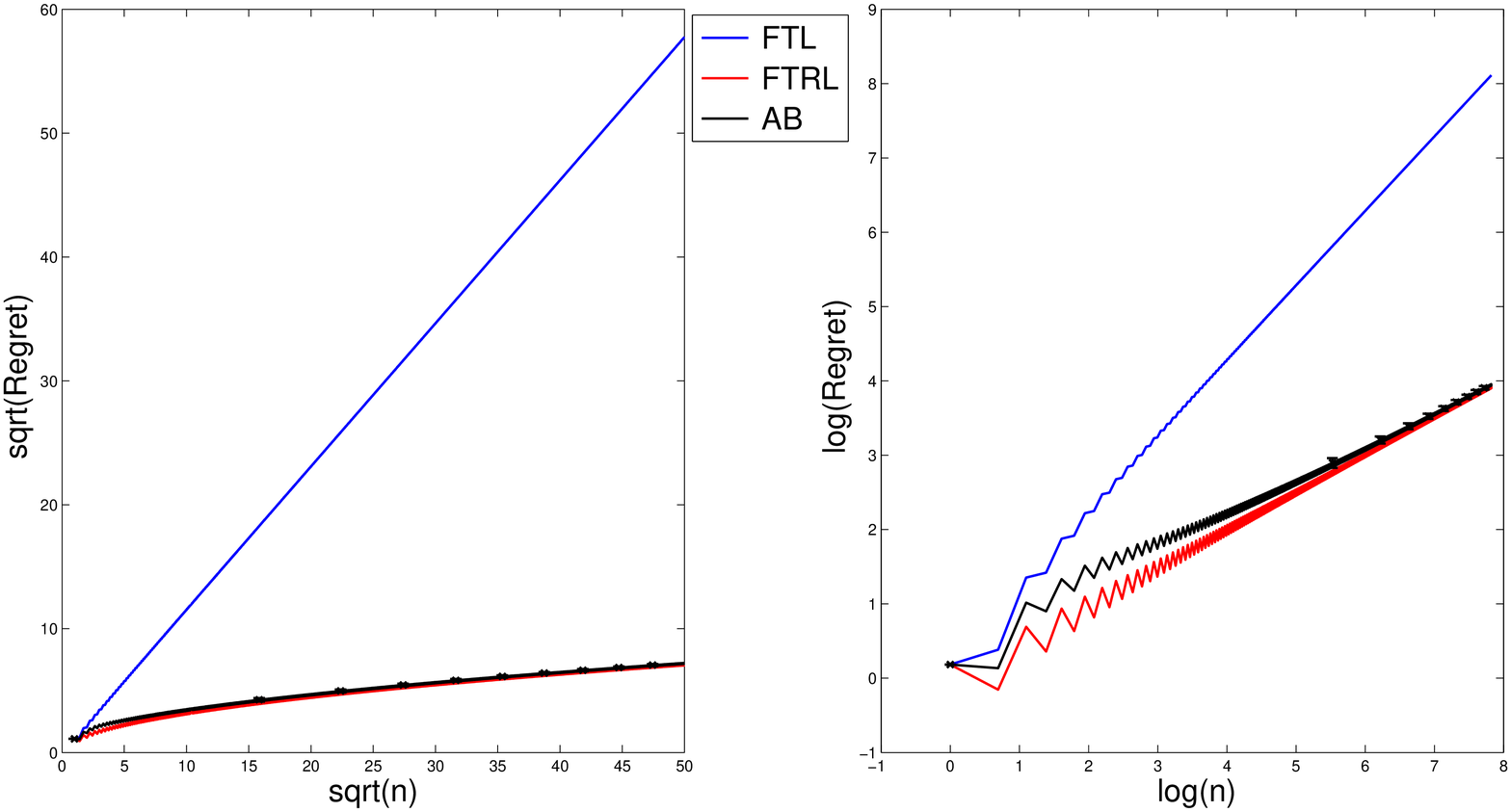}	
	\caption{Experimental results for worst-case data. \label{res:worst_case}}
\end{figure}

\paragraph{The unit ball}
We close this section by comparing the performance of our adaptive algorithms on the unit ball, namely, FTL, FTSL, FTLR, and AB(FTL,FTRL). All these algorithms are parametrized as above. The problem setup is similar to the stochastic data setting and the worst-case data setting. Again, we consider a 4-dimensional setting, that is, $\cW$ is the unit ball in $\R^4$ centered at the origin.
The worst-case data is generated exactly as above, while the generation process of the stochastic data is slightly modified to increase the difference between FTLR and FTL: we sample the i.i.d. vectors $\hat{f}_t$ from a zero-mean normal distribution with independent components whose variance is $1/16$, and let $\tilde{f_t}=\hat{f}_t$ if $\|\hat{f}_t\|_2 \le 1$ and $\tilde{f}_t = \hat{f}_t/\norm{\hat{f}_t}_2$ when $\norm{\hat{f}_t}_2>1$ (i.e., we only normalize if $\hat{f}_t$ falls outside of the unit ball).
The reason of this modification is to encourage the occurrence of the event $\|F_{t-1}\|_2 < \sqrt{t-1}$. Recall that when $\|F_{t-1}\|_2 \ge \sqrt{t-1}$, the prediction of FTRL matches that of FTL, so we are trying to create some data where their behavior is actually different. As a result, we will be able to observe that the predictions of FTL and FTRL are different in the early rounds. Finally, as before, we let $f_t=\tilde{f}_t + L e_1$, and set the time horizon to $n=20,000$.

The results of the simulation of the stochastic data setting are shown in Figure~\ref{res:Stoc_unitBall}. In the case of $L=0.1$, FTRL suffers more regret at the beginning for some rounds, but then succeeds to match the performance of FTL.
The results of the simulation of the worst-case data setting are shown in Figure~\ref{res:WorstCase_unitBall}, where FTSL has similar performance as FTRL.
\begin{figure}[h]
	\centering
	\includegraphics[width=0.8\textwidth]{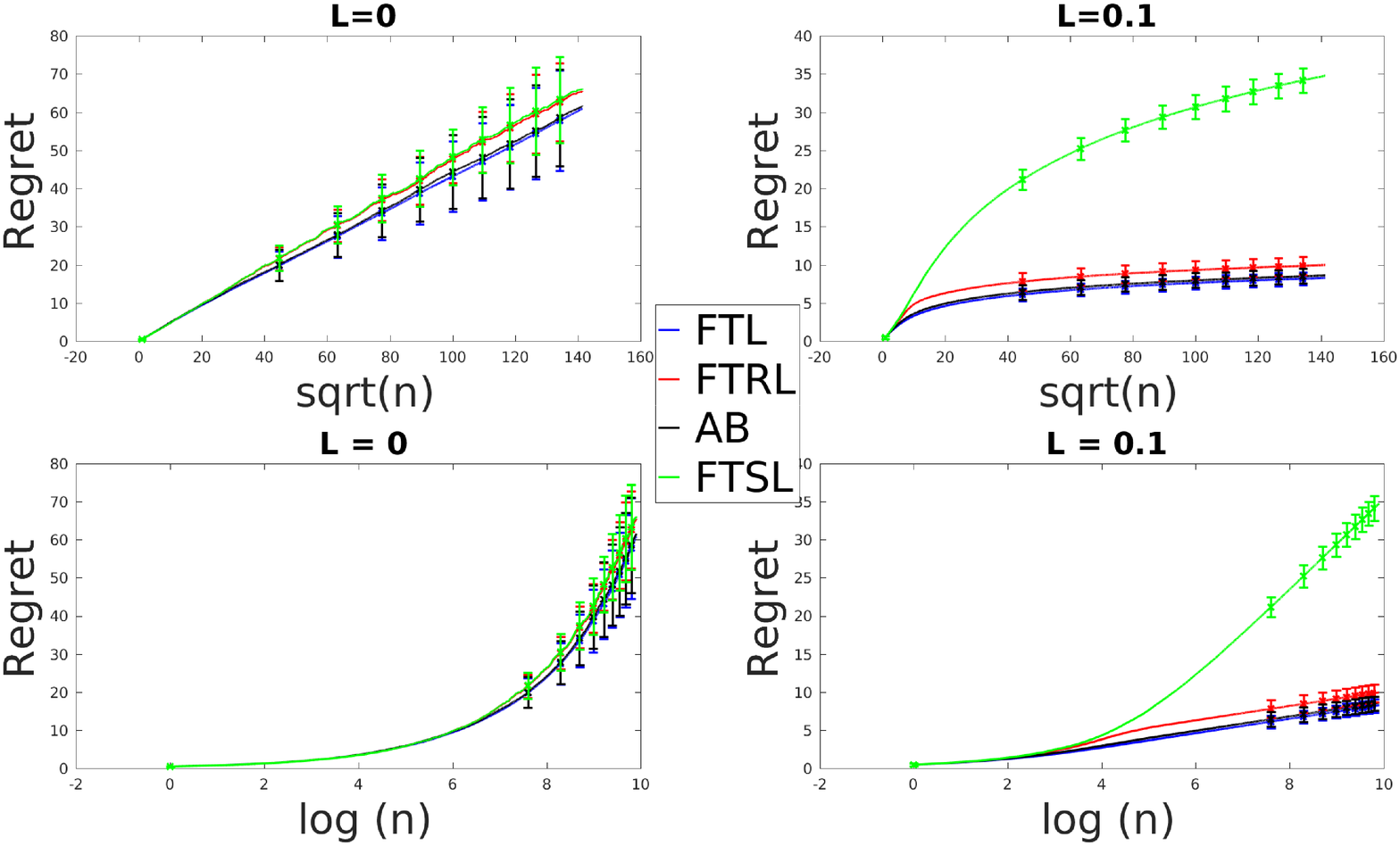}	
	\caption{Experimental results for stochastic data when $\cW$ is the unit ball. \label{res:Stoc_unitBall}}
\end{figure}

\begin{figure}[h]
	\centering
	\includegraphics[width=0.8\textwidth]{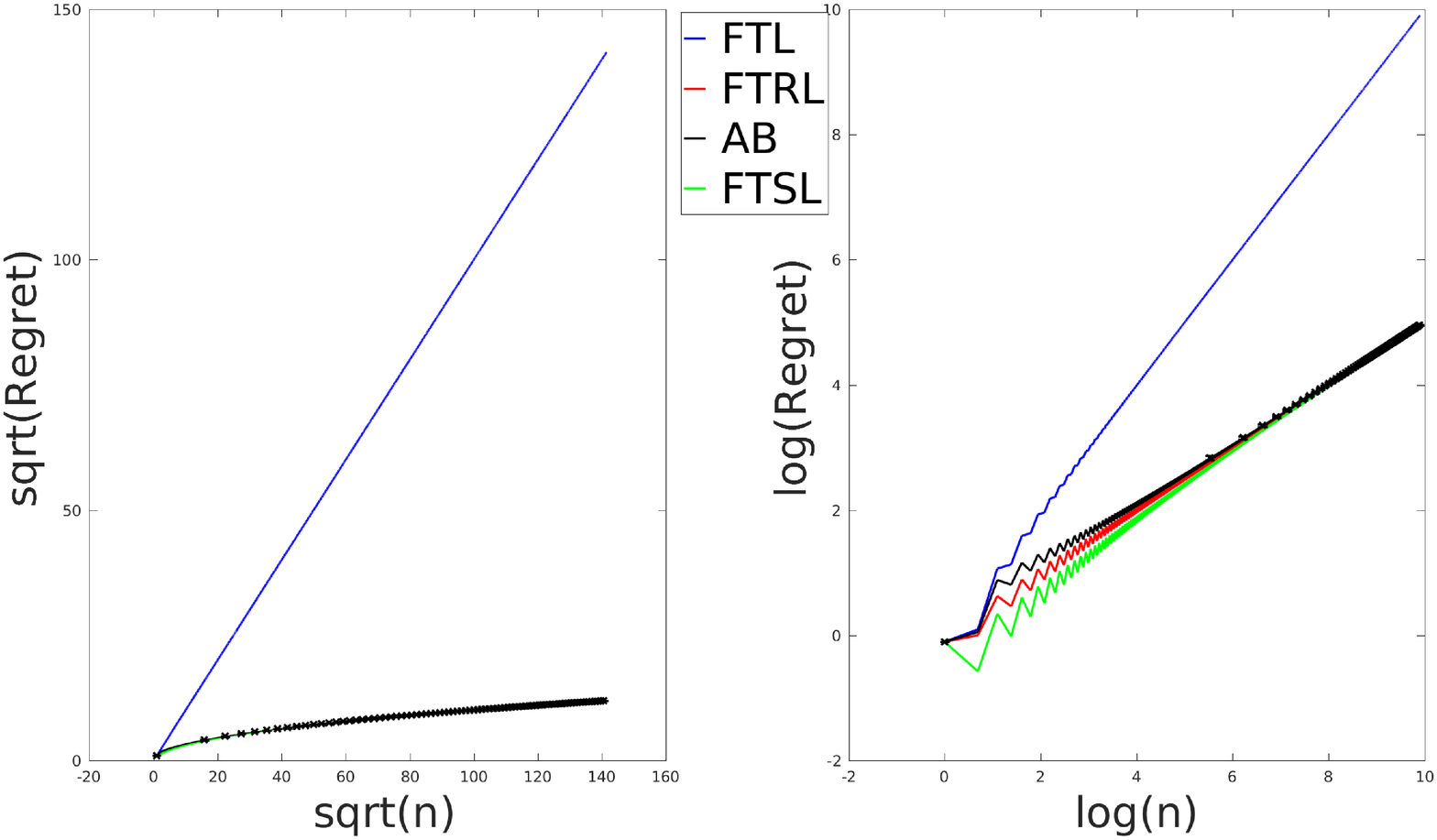}	
	\caption{Experimental results for worst-case data when $\cW$ is the unit ball. \label{res:WorstCase_unitBall}}
\end{figure}

\section{Conclusion}
FTL is a simple method that is known to perform well in many settings, while 
existing worst-case results fail to explain its good performance.
While taking a thorough look at why and when FTL can be expected to achieve small regret, we discovered that the curvature of the boundary of the constraint and having average loss vectors bounded away from zero help keep the regret of FTL small. These conditions are significantly different from previous conditions on the curvature of the loss functions which have been considered extensively in the literature.
It would be interesting to further investigate this phenomenon for other algorithms or in other learning settings.
%We suspect that other algorithms can also take advantage of this, opening up interesting new 
%Amongst the numerous possible follow-up questions, perhaps the most interesting would be to design 
%a version of FTRL where the regularization coefficient would be chosen adaptive to interpolate between the performance
%of FTL and FTRL in a data-dependent fashion. While such a result was achieved with 
%($\cA$, $\cB$)-prod, 

\appendix
\section{Appendix: Technical results}

\subsection{Strongly convex sets and principal curvatures}
Recall that a convex set $\cW\subset \R^d$ is $\lambda$-strongly convex if for any $x,y\in \cW$, $\gamma\in [0,1]$, $\cW$ contains the ball of center $\gamma x + (1-\gamma) y$ that has a radius of $\gamma(1-\gamma) \frac{\lambda}{2} \norm{x-y}^2$.
That is, for any $z\in \R^d$ with $\norm{z}=1$, $\gamma x + (1-\gamma) y + \gamma(1-\gamma) \frac{\lambda}{2} \norm{x-y}^2 z\in \cW$.
Let $B_r(x)=\set{y \in \R^d}{ \|x-y\|_2 \le r}$ denote the Euclidean ball of radius $r$ centered at $x$.

\begin{prop}
\label{strongconvex}
Let $\cW \subset \R^d$ be a $C^2$ convex body with support function $\varphi$, %and assume that all principal curvatures of $\cW$ are non-zero 
%\todoa{Can we rephrase this without the curvature? E.g., using $\varphi$?} \todor{Prove the equivalence by (i)->(ii)->(iii)->(i), then this condition can be removed.} 
and let $\lambda$ be an arbitrary positive number.
Then the following statements are equivalent:
\begin{enumerate}[(i)]
\item \label{sc:l1} The smallest principal curvature of $\cW$ is at least $\lambda$.
\item \label{sc:l2} $\cW= {\large \cap}_{\theta \in \bS^{d-1} }{B_{1/\lambda} (w_\theta- \theta /\lambda)}$ where $w_\theta \in \partial \varphi(\theta) \subset \bd(\cW)$.
\item \label{sc:l3} $\cW$ is $\lambda$-strongly convex.
\end{enumerate}
\end{prop}
Condition \eqref{sc:l2}, which is actually the definition of \citet{Pol96} for strongly convex sets, means that $\cW$ can be obtained as the intersection of closed balls of radius $1/\lambda$, such that there is one ball for every boundary point $w$ and  tangent hyperplane $P$ where the ball touches $P$ in $w$. Note that a ball with radius $1/\lambda$ satisfies all conditions: \eqref{sc:l1} and \eqref{sc:l2} by definition, while \eqref{sc:l3} holds, e.g., by Example~13 of \citet{JourneeNRS10}.

\begin{proof}
We show that \eqref{sc:l1} implies \eqref{sc:l2}, \eqref{sc:l2} implies \eqref{sc:l3}, and \eqref{sc:l3} implies \eqref{sc:l1}. 

We start with showing that \eqref{sc:l1} implies \eqref{sc:l2}.
First note that all principal curvatures of the $d$-dimensional ball $B=B_{1/\lambda}(0)$ with radius $1/\lambda$ (centered at the origin) are $\lambda$. Therefore, \eqref{sc:l1} and Theorem~3.2.9 of \citet{Sch14:ConvexBodies} implies that there is a convex body $\cM$ such that $\cW+\cM=B$, where for two sets, $S_1, S_2 \subset \R^d$, $S_1+S_2$ is defined as $\set{s_1+s_2}{s_1 \in S_1, s_2\in S_2}$. For any $\theta \in \bS^{d-1}$, let $m_\theta \in \argmax_{m \in \cM} \ip{m,\theta}$. Then clearly $w_\theta+m_\theta$ maximizes $\ip{b,\theta}$ for $b \in \cW+\cM$. Therefore, $\cW + m_\theta$ is a subset of $B$ and touches it at $w_\theta+m_\theta$, or equivalently $\cW \subset B-m_\theta$ and they touch each other, and a tangent hyperplane with normal vector $\theta$, in $w_\theta$. This proves that \eqref{sc:l1} implies \eqref{sc:l2}. 
%The other direction is a trivial consequence of Corollary~3.2.10 of \citet{Sch14:ConvexBodies}.

Next we prove that  \eqref{sc:l2} implies \eqref{sc:l3}. Assuming \eqref{sc:l2} holds, let $w \in \cW$ be any point in the interior of $\cW$, and let $p \in \bd(\cW)$ be the closest boundary point to $w$, and recall that $T_p\cW$ is the tangent space of $\cW$ at $p$. By construction, $B_{\|w-p\|_2}(w)$ touches the boundary of $\cW$ at $p$ (in the sense that they do not intersect, but they can have multiple common points), and so $w-p$ is orthogonal to $T_p\cW$. Therefore,  $B_{\|w-p\|_2}(w)$ also touches the boundary of the ball $B=B_{1/\lambda}(p+\frac{w-p}{\lambda \|w-p\|_2})$, which contains $\cW$ by assumption \eqref{sc:l2}. Now consider any two points $x,y \in \cW$ and $\gamma \in [0,1]$ such that $w=\gamma x + (1-\gamma) y$. Then the ball with radius $\lambda \gamma (1-\gamma) \|x-y\|_2^2/2$ centered at $w$ is contained in $B$, since $B$ is $\lambda$-strongly convex. But then its radius is at most $\|p-w\|_2$, and so it is also contained in $\cW$. This shows that $\cW$ is $\lambda$-strongly convex, thus  \eqref{sc:l3} holds.

To finish the proof of the proposition, assume \eqref{sc:l3}. To prove that \eqref{sc:l1} holds, we have to show, that for any point $p$ on $\bd(\cW)$ and for any unit vector $v \in T_p\cW$, the curvature of the boundary along $v$ is at least $\lambda$.
Let $P$ be the hyperplane spanned by $v$ and the outer normal vector $u$ of $\cW$ at point $p$, and consider the planar curve $\gamma$ defined by $\bd(\cW) \cap P$.
Using $v$ as the axis of a local coordinate system, a point $w(s)$ on the curve $\gamma$ in the neighborhood of $p$ can be expressed as $w(s) = p + sv - f(s)u$ for an appropriate function $f$, as illustrated in \cref{fig:stronglyconvexset}.
\begin{figure}[t]
\centering
	\includegraphics[width=0.7\textwidth]{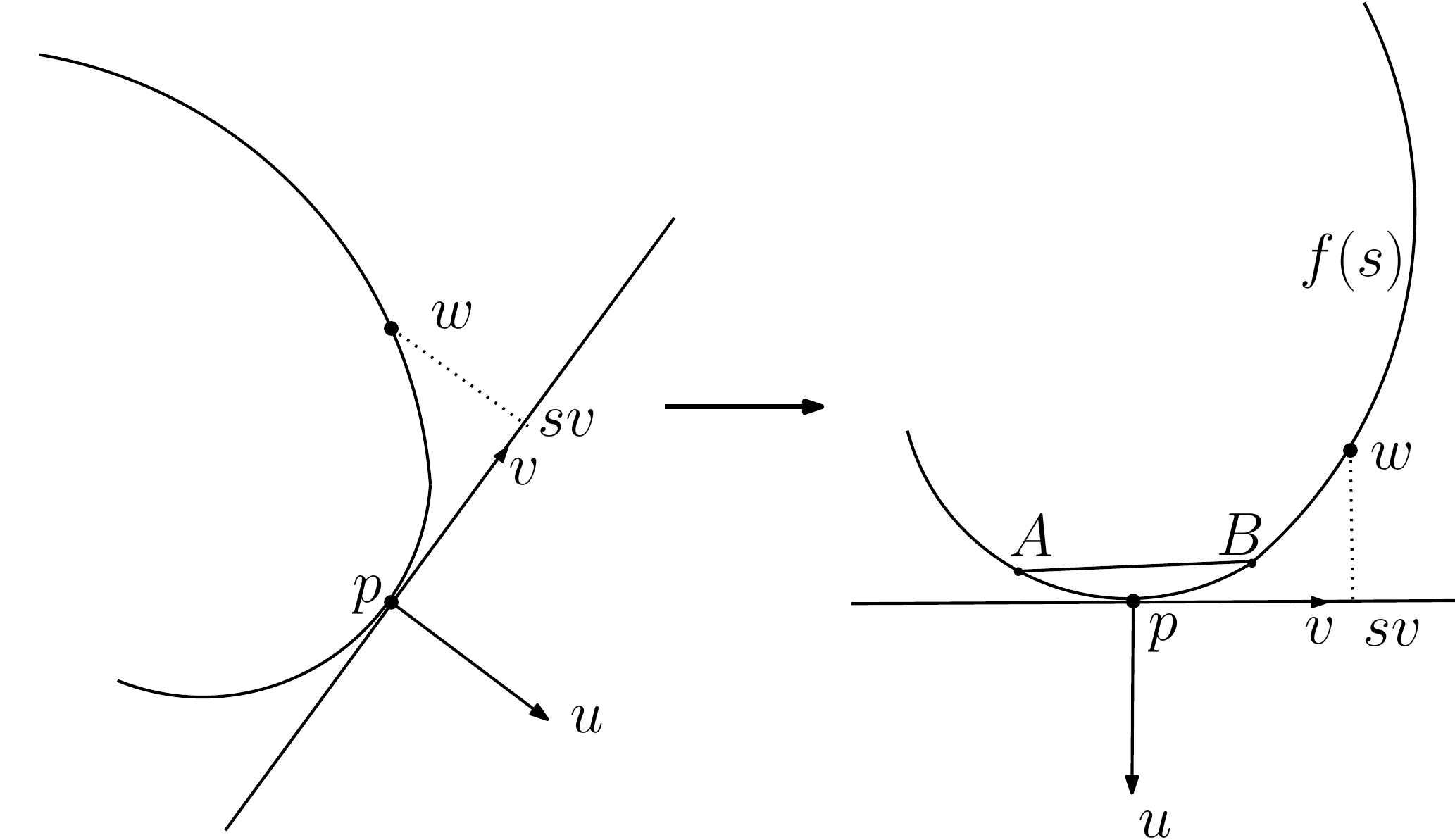}
	\caption{The local coordinate system at $p$. \label{fig:stronglyconvexset}}
\end{figure} 

Note that $f'(0)=0$, and by Proposition 2.1 of \citet{pressley2010elementary}, the curvature of $\gamma$ at $p$ can be obtained as 
\[
\frac{f''(s)}{\sqrt{1+f'(s)^2}^3}\Bigg\vert_{s=0} = f''(0)~.
\]
Now since $w(s),w(-s) \in \cW$ for a sufficiently small $s$, the strong convexity of $\cW$ applied to $w(s)$ and $w(-s)$ with $\gamma=1/2$
implies that $q=\frac{w(s)+w(-s)}{2}+\frac{\lambda}{8} \|w(s)-w(-s)\|_2^2 u \in \cW$. Substituting the definition of $w(s)$ and $w(-s)$, we get
\[
q=p - u \left[\frac{f(s)+f(-s)}{2} - \frac{\lambda}{8} \Bigl(4s^2 + (f(s)-f(-s))^2\Bigr)\right].
\]
Therefore, $q \in \cW$ implies
$f(s)+f(-s) \ge \lambda s^2$, and so
\[
	f''(0) = \lim_{s \ra 0} \frac{\frac{f(s) - f(0)}{s} - \frac{f(0) - f(-s)}{s}}{s} = \frac{f(s)+ f(-s)}{s^2} \ge \lambda.
\]
Thus \eqref{sc:l1} holds, finishing the proof of the proposition.
\end{proof}

\subsection{Proof of \cref{prop:derivativePhi}}
Under the extra condition that $\cW$ is compact
the result follows from Danskin's theorem (e.g., Proposition B.25 of \citealt{bertsekas99nonlinear}).
However, compactness is not required. \todoc{In fact, I don't get why Bertsekas needs it for this part of the statement. He may need it elsewhere.}
For completeness, we provide a short, direct proof. 
\todoc{Probably delegate to the appendix.}
We need to show that 
$\cZ = \partial \varphi(\Theta)$ where recall that
\begin{align*}
\partial \varphi(\Theta)= \set{u\in \R^d}{ \varphi(\Theta) + \inpro{u}{\cdot-\Theta} \le \varphi(\cdot)}
= \set{u\in  \R^d}{ \varphi(\Theta)  \le \inpro{u}{\Theta} + \varphi(\cdot) - \inpro{u}{\cdot} }\,.
\end{align*}
Since $\cZ \subset \cW$, 
if $w\in \cZ$, $\varphi(\Theta') \ge \ip{w,\Theta'}$ for any $\Theta'$ by the definition of $\varphi$.
Hence, $\varphi(\Theta) = \ip{w,\Theta} \le \ip{w,\Theta} + \varphi(\Theta')-\ip{w,\Theta'}$ for any $\Theta'$, implying that $w\in \partial \varphi(\Theta)$.

On the other hand, assume $w \in \partial \varphi(\Theta)$. Then $\varphi(\Theta) \le \ip{w,\Theta}$
since $\varphi(0) = \ip{w,0} = 0$. 
%Now notice that  $\varphi(\lambda \Theta) = \lambda \varphi(\Theta)$ holds for any $\lambda>0$.
%Using this and that  $\varphi(\Theta') \ge \varphi(\Theta) + \ip{w,\Theta'-\Theta}$ with $\Theta'=2\Theta$,
%we get $\ip{w,\Theta}\le \varphi(\Theta)$. Hence, $\ip{w,\Theta} = \varphi(\Theta)$.
Since $\cW$ is closed, $\cZ$ is also closed. Therefore, if $w \not\in\cZ$,
the strict separation theorem (applied to $\{w\}$, a convex compact set,
and $\cZ$, a convex closed set) implies that
 there exists $\rho\in \R^d$ such that $\ip{z,\rho} < \ip{w,\rho}$ for all $z\in \cZ$.
 Let $\Theta' = \Theta +  \rho$.
 Then, $\varphi(\Theta') = \max_{u\in \cW} \ip{u,\Theta} +  \ip{u,\rho}
 < \varphi(\Theta) +  \ip{w,\Theta'-\Theta} \le \ip{w,\Theta'} \le \varphi(\Theta')$, a contradiction.
Hence, $w\in \cZ$.

\subsection{Technical lemmas for the lower bound \cref{thm:lowerbound}}

\begin{lemma}[Concentration of $\hP_{t}$] For any $u>0$,
	\label{lem:concenPhat}
	\[
	\Probc{|\hat P_{t}-P| > \frac{K}{2K+t}|1-2P| + \frac{t}{2K+t}u}{P} \le 2\exp(-tu^2)~.
	\]
\end{lemma}
\begin{proof}
	Recall that $\hat P_t = \frac{K+\sum_{i=1}^{t}X_i}{2K+t}$. Thus, 
	\begin{align} 
	\Probc{|\hat P_{t}-P| > u }{P} & = \Probc{\left| \frac{K+\sum_{i=1}^{t}X_i}{2K+t} -P\right| > \frac{K}{2K+t}|1-2P| + \frac{t}{2K+t}u}{P} \nonumber \\ 
	& =  \Probc{\left| \sum_{i=1}^{t}X_i - Pt + K(1-2P) \right| > K|1-2P|+ tu }{P} \nonumber \\ 
	& \le \Probc{\left| \sum_{i=1}^{t}X_i - Pt \right| > tu }{P}, \label{eq:hoeffding}
	\end{align}
	where the last inequality is due to $\Prob{|A+b|>c} \le \Prob{|A| > c-|b|}$. Note that conditioned on $P$, $X_1, \ldots, X_t$ are independent Bernoulli random variables with expectation $P$, thus \eqref{eq:hoeffding} holds by Hoeffding's inequality (see, e.g., \cite[Corollary~A.1]{CBLu06:book}).
\end{proof}

%\begin{lemma}[Length of $\theta_t$]
%	\label{lem: lengthoftheta}
%	\[
%	L \le \E \left[\|\Theta_t\|_2 \right] \le L+\frac{1}{4K}.
%	\] 
%\end{lemma}
%\begin{proof}
%	Note that the second component of $\theta_t$ is $-L$, thus $\E \left[\|\theta\|_2 \right] \ge L$. For the other inequality, note that 
%	\[
%	\E \left[ \| \frac{1}{t} \sum_{i=1}^{t} f_i\|_2\right] \le \E \left[ \| f_i\|_2\right] = \E[\sqrt{L^2 + (1-2X_i)^2}] \le L + \sqrt{\E[|1-2X_i|^2]} = L+\frac{1}{4K}.
%	\]
%\end{proof}

\begin{lemma}
	\label{lem:bayeserror}
	\[
	\Expc{(P-\hat{P}_t)^2}{P} = \frac{K^2(1-2P)^2}{(2K+t)^2} + \frac{tP(1-P)}{(2K+t)^2}.
	\]
\end{lemma}
\begin{proof}
	Recall that $\hat P_t = \frac{K+\sum_{i=1}^{t}X_i}{2K+t}$.Thus, 
	\begin{align*}
	\Expc{(P-\hat{P}_t)^2}{P} & = \Expc{\left(\frac{K(1-2P)}{2K+t} + \frac{\sum_{i=1}^{t}X_i- Pt}{2K+t}\right)^2}{P} \\
	& = \frac{K^2(1-2P)^2}{(2K+t)^2} + \frac{1}{(2K+t)^2}\Expc{ \left(\sum_{i=1}^{t}X_i - tP\right)^2}{P} \\
	& = \frac{K^2(1-2P)^2}{(2K+t)^2} + \frac{tP(1-P)}{(2K+t)^2},
	\end{align*}
	where the second equality is due to $\Expc{ \sum_{i=1}^{t}X_i - Pt}{P} =0$, and the last equality is due to that conditioned on $P$, $\sum_{i=1}^{t}X_i$ has a Binomial distribution with parameters $t$ and $P$.
\end{proof}
\begin{lemma} Under the assumptions of \cref{thm:lowerbound}, for any $0<P_1,P_2<1$,
	\label{lem:P2P1loss}
	\[
	\inner{w^{P_2} - w^{P_1}, f^{P_1}} \ge \frac{hL}{2} \frac{\left( \frac{2P_2 - 2P_1}{hL} \right)^2}{\sqrt{1+\left( \frac{1-2P_1}{hL}\right)^2 } \left(1+\left( \frac{1-2P_2}{hL}\right)^2 \right)}.
	\]
\end{lemma}
\begin{proof}
	It is easy to see that for any $p$, $w^p$ is on the boundary of $\cW$, that is, $w^p = \argmin_{w\in\cW} \inner{ w, f^p } = (\cos (\varphi^p), h\sin (\varphi^p))$ for some $\varphi^p$. 
	Then $\inner{w^p,f^p}= (2p-1) \cos (\varphi^p) - Lh \sin (\varphi^p)$, and so taking the derivative it is easy to verify that $\tan(\varphi^p) = \frac{Lh}{1-2p}$ and $\sin(\varphi^p) = \frac{Lh}{\sqrt{(Lh)^2+(1-2p)^2}} >0$. 
	Thus, $1-2P_1 = \frac{Lh\cos (\varphi^{P_1})}{\sin (\varphi^{P_1})}$. To simplify notation, let $\varphi_1=\varphi^{P_1}$ and $\varphi_2=\varphi^{P_2}$. Then,
	\begin{align}
	\langle w^{P_2} - w^{P_1}, f^{P_1} \rangle & = \left\langle \left(
	\begin{array}{c}
	\cos \varphi_{2} - \cos \varphi_{1} \\
	h\left(\sin \varphi_{2} - \sin \varphi_{1} \right) 
	\end{array}
	\right),  \left( 
	\begin{array}{c}
	\frac{-hL\cos \varphi_{1}}{\sin \varphi_{1}} \\
	-L
	\end{array}
	\right) \right\rangle \nonumber \\ 
	& = -hL\left( \left( \cos (\varphi_{2}) - \cos (\varphi_{1} ) \right) \frac{\cos (\varphi_{1})}{\sin( \varphi_{1})} 
		 + \left(\sin (\varphi_{2}) - \sin (\varphi_{1} )\right)  \right)  \nonumber \\
	& = \frac{-hL}{\sin( \varphi_{1})} \left( \cos (\varphi_{2} )\cos (\varphi_{1} )- \cos^2(\varphi_{1} )+ \sin (\varphi_{1}) \sin (\varphi_{2}) - \sin^2(\varphi_{1}) \right)  \nonumber \\
	& = \frac{hL}{\sin (\varphi_{1})} \left( 1- \cos (\varphi_{2}) \cos (\varphi_{1})  - \sin (\varphi_{1}) \sin (\varphi_{2}) \right) \nonumber \\
	& = \frac{hL}{\sin (\varphi_{1})} \left( 1- \cos(\varphi_{1}-\varphi_{2}) \right) \nonumber \\
	& = \frac{hL}{\sin (\varphi_{1})} \left(\frac{1}{2}\left(\cos(\varphi_{1}-\varphi_{2})-1\right)^2 + \frac{1}{2} \sin^2 (\varphi_{1}-\varphi_{2})\right) \\
	& \ge  \frac{hL}{2 \sin (\varphi_{1})}  \sin^2 (\varphi_{1}-\varphi_{2}) \nonumber \\
	& = \frac{hL}{2} \sin (\varphi_{1}) \sin^2 \varphi_{2} \left(\cot (\varphi_{1} )- \cot (\varphi_{2})\right)^2~. % \\
%	& = \frac{hL}{2}  \frac{\left( \frac{2P_2 - 2P_1}{hL} \right)^2}{\sqrt{1+\left( \frac{1-2P_1}{hL}\right)^2 } \left(1+\left( \frac{1-2P_2}{hL}\right)^2\right)},
	\end{align}
	The proof is finished by substituting $\cot (\varphi_i) = \frac{1-2P_i}{hL}$, $\sin(\varphi_1) = \frac{1}{\sqrt{1+\left(\frac{1-2P_1}{Lh}\right)^2}}$ and $\sin^2 (\varphi_2) =   \frac{1}{1+\left(\frac{1-2P_2}{Lh}\right)^2}$.
\end{proof}

\section*{Acknowledgements}
This work was supported in part by the Alberta Innovates Technology Futures through the Alberta Ingenuity Centre for Machine Learning and by NSERC.
During part of this work, T. Lattimore was with the Department of Computing Science, University of Alberta.

%\small
\setlength{\bibsep}{0.4\bibsep}
\bibliography{reference}
\bibliographystyle{plainnat}
\end{document}